\documentclass[accepted]{uai2021} 

\usepackage[american]{babel}

\usepackage{natbib} 

\usepackage{mathtools} 
\usepackage{siunitx} 
\usepackage{booktabs} 
\usepackage{tikz} 
\usepackage{xcolor}
\usepackage{pifont}
\usepackage{amsmath}
\usepackage{amsthm}
\usepackage{amssymb}
\usepackage{marginnote}
\usepackage{lipsum}
\usepackage{multicol}
\usepackage{multirow}
\usepackage{adjustbox}
\usepackage{stackrel}
\usepackage[para]{threeparttable}
\usepackage{tikz}  
\usepackage{verbatim}
\usepackage[bottom]{footmisc}
\usepackage{tablefootnote}
\usepackage{xr}
\usepackage{cleveref}

\DeclareMathOperator*{\argmax}{argmax}

\DeclareMathOperator{\bbN}{\mathbb{N}}
\DeclareMathOperator{\bbR}{\mathbb{R}}

\DeclareMathOperator{\calC}{\mathcal{C}}
\DeclareMathOperator{\calD}{\mathcal{D}}
\DeclareMathOperator{\calE}{\mathcal{E}}
\DeclareMathOperator{\calG}{\mathcal{G}}
\DeclareMathOperator{\calH}{\mathcal{H}}

\DeclareMathOperator{\calV}{\mathcal{V}}
\DeclareMathOperator{\calX}{\mathcal{X}}

\DeclareMathOperator{\bfa}{\mathbf{a}}

\DeclareMathOperator{\bfc}{\mathbf{c}}

\DeclareMathOperator{\bfv}{\mathbf{v}}
\DeclareMathOperator{\bfx}{\mathbf{x}}
\DeclareMathOperator{\bfy}{\mathbf{y}}

\DeclareMathOperator{\bfalpha}{\mathbf{\alpha}}
\DeclareMathOperator{\bfbeta}{\mathbf{\beta}}
\DeclareMathOperator{\bftheta}{\mathbf{\theta}}

\newtheorem{theorem}{Theorem}[section]
\newtheorem{corollary}{Corollary}[theorem]

\newtheorem{proposition}{Proposition}

\theoremstyle{definition}
\newtheorem{definition}{Definition}[section]

\theoremstyle{remark}
\newtheorem*{remark}{Remark}

\title{Mixed Variable Bayesian Optimization with Frequency Modulated Kernels}

\author[1]{\href{mailto:Changyong Oh <c.oh@uva.nl>}{Changyong Oh}{}} 
\author[1]{Efstratios Gavves}
\author[1,2]{Max Welling}
\affil[1]{%
    Informatics Institute\\
    University of Amsterdam\\
    Amsterdam, The Netherlands
}
\affil[2]{%
    Qualcomm AI Research Netherlands\\
    Amsterdam, The Netherlands
}

\makeatletter
\newcommand*{\addFileDependency}[1]{
  \typeout{(#1)}
  \@addtofilelist{#1}
  \IfFileExists{#1}{}{\typeout{No file #1.}}
}
\makeatother

\newcommand*{\myexternaldocument}[1]{%
    \externaldocument{#1}%
    \addFileDependency{#1.tex}%
    \addFileDependency{#1.aux}%
}

\myexternaldocument{oh_359-supp}

\begin{document}
\maketitle

\begin{abstract}
The sample efficiency of Bayesian optimization~(BO) is often boosted by Gaussian Process~(GP) surrogate models.
However, on mixed variable spaces, surrogate models other than GPs are prevalent, mainly due to the lack of kernels which can model complex dependencies across different types of variables.
In this paper, we propose the frequency modulated (FM) kernel flexibly modeling dependencies among different types of variables, so that BO can enjoy the further improved sample efficiency.
The FM kernel uses distances on continuous variables to modulate the graph Fourier spectrum derived from discrete variables.
However, the frequency modulation does not always define a kernel with the similarity measure behavior which returns higher values for pairs of more similar points. 
Therefore, we specify and prove conditions for FM kernels to be positive definite and to exhibit the similarity measure behavior.
In experiments, we demonstrate the improved sample efficiency of GP BO using FM kernels~(BO-FM).
On synthetic problems and hyperparameter optimization problems, BO-FM outperforms competitors consistently.
Also, the importance of the frequency modulation principle is empirically demonstrated on the same problems.
On joint optimization of neural architectures and SGD hyperparameters, BO-FM outperforms competitors including Regularized evolution~(RE) and BOHB.
Remarkably, BO-FM performs better even than RE and BOHB using three times as many evaluations.

\end{abstract}

\section{Introduction}
\label{sec:introduction}
\vspace{-4pt}
\vspace{-4pt}

Bayesian optimization has found many applications ranging from daily routine level tasks of finding a tasty cookie recipe~\citep{solnik2017bayesian} to sophisticated hyperparameter optimization tasks of machine learning algorithms (e.g. Alpha-Go~\citep{chen2018bayesian}). 
Much of this success is attributed to the flexibility and the quality of uncertainty quantification of Gaussian Process~(GP)-based surrogate models~\citep{snoek2012practical,swersky2013multi,oh2018bock}.
 
Despite the superiority of GP surrogate models, as compared to non-GP ones, their use on spaces with discrete structures (\emph{e.g.,} chemical spaces~\citep{reymond2012exploring}, graphs and even mixtures of different types of spaces) is still application-specific~\citep{kandasamy2018neural,korovina2019chembo}. 
The main reason is the difficulty of defining kernels flexible enough to model dependencies across different types of variables.
On mixed variable spaces which consist of different types of variables including continuous, ordinal and nominal variables, current BO approaches resort to non-GP surrogate models, such as simple linear models or linear models with manually chosen basis functions~\citep{daxberger2019mixed}.
However, such linear approaches are limited because they may lack the necessary model capacity.
 
There is much progress on BO using GP surrogate models~(GP BO) for continuous, as well as for discrete variables. However, for mixed variables it is not straightforward how to define kernels ,which can model dependencies across different types of variables. 
To bridge the gap, we propose \textit{frequency modulation} which uses distances on continuous variables to modulate the frequencies of the graph spectrum ~\citep{ortega2018graph} where the graph represents the discrete part of the search space~\citep{oh2019combinatorial}.

A potential problem in the frequency modulation is that it does not always define a kernel with the similarity measure behavior~\citep{vert2004primer}.
That is, the frequency modulation does not necessarily define a kernel that returns higher values for pairs of more similar points.
Formally, for a stationary kernel $k(x,y) = s(x - y)$, $s$ should be decreasing~\citep{remes2017non}.
In order to guarantee the similarity measure behavior of kernels constructed by frequency modulation, we stipulate a condition, the \textit{frequency modulation principle}.
Theoretical analysis results in proofs of the positive definiteness as well as the effect of the frequency modulation principle.
We coin frequency modulated~(FM) kernels as the kernels constructed by frequency modulation and respecting the frequency modulation principle.

Different to methods that construct kernels on mixed variables by kernel addition and kernel multiplication, for example, FM kernels do not impose an independence assumption among different types of variables.
In FM kernels, quantities in the two domains, that is the distances in a spatial domain and the frequencies in a Fourier domain, interact.
Therefore, the restrictive independence assumption is circumvented, and thus flexible modeling of mixed variable functions is enabled.

In this paper, \textit{(i)} we propose frequency modulation, a new way to construct kernels on mixed variables, \textit{(ii)} we provide the condition to guarantee the similarity measure behavior of FM kernels together with a theoretical analysis, and \textit{(iii)} we extend frequency modulation so that it can model complex dependencies between arbitrary types of variables.
In experiments, we validate the benefit of the increased modeling capacity of FM kernels and the importance of the frequency modulation principle for improved sample efficiency on different mixed variable BO tasks.
We also test BO with GP using FM kernels~(BO-FM) on a challenging joint optimization of the neural architecture and the hyperparameters with two strong baselines, Regularized Evolution~(RE)~\citep{real2019regularized} and BOHB~\citep{falkner2018bohb}.
BO-FM outperforms both baselines which have proven their competence in neural architecture search~\citep{dong2021nats}. 
Remarkably, BO-FM outperforms RE with three times evaluations.

\section{Preliminaries}
\label{sec:preliminary}
\vspace{-4pt}

\subsection{Bayesian Optimization with Gaussian Processes}
\vspace{-4pt}
Bayesian optimization~(BO) aims at finding the global optimum of a black-box function $g$ over a search space $\calX$. 
At each round BO performs an evaluation $y_i$ on a new point $\bfx_i \in \calX$, collecting the set of evaluations $\calD_t=\{(\bfx_i, y_i)\}_{i=1,\cdots,t}$ at the $t$-th round.
Then, a surrogate model approximates the function $g$ given $\calD_t$ using the predictive mean $\mu(\bfx_* \vert \calD_t)$ and the predictive variance $\sigma^2 (\bfx_* \vert \calD_t)$.
Now, an acquisition function $r(\bfx_*) = r(\mu(\bfx_* \vert \calD_t), \sigma^2 (\bfx_* \vert \calD_t))$ quantifies how informative input $\bfx \in \calX$ is for the purpose of finding the global optimum.
$g$ is then evaluated at $\bfx_{t+1} = \argmax_{\bfx \in \calX} r(\bfx)$, $y_{t+1} = g(\bfx_{t+1})$.
With the updated set of evaluations, $\calD_{t+1} = \calD_t \cup \{(\bfx_{t+1}, y_{t+1})\}$, the process is repeated.

A crucial component in BO is thus the surrogate model.
Specifically, the quality of the predictive distribution of the surrogate model is critical for balancing the exploration-exploitation trade-off~\citep{shahriari2015taking}. Compared with other surrogate models~(such as Random Forest~\citep{hutter2011sequential} and a tree-structured density estimator~\citep{bergstra2011algorithms}), Gaussian Processes~(GPs) tend to yield better results~\citep{snoek2012practical,oh2018bock}.

For a given kernel $k$ and data $\calD=(\mathbf{X},\bfy)$ where $\mathbf{X}=[\bfx_1, \cdots, \bfx_n]^T$ and $\bfy = [y_1, \cdots, y_n]^T$, a GP has a predictive mean $\mu(\bfx_* \vert \mathbf{X},\bfy) = k_{*\mathbf{X}}(k_{\mathbf{X}\mathbf{X}} + \sigma^2 I)^{-1}\bfy$ and predictive variance $\sigma^2 (\bfx_* \vert \mathbf{X},\bfy) = k_{**} - k_{*\mathbf{X}}(k_{\mathbf{X}\mathbf{X}} + \sigma^2 I)^{-1}k_{\mathbf{X}*}$ where $k_{**} = k(\bfx_*, \bfx_*)$, $[k_{*\mathbf{X}}]_{1,i} = k(\bfx_*, \bfx_i)$, $k_{\mathbf{X}*} = (k_{*\mathbf{X}})^T$ and $[k_{\mathbf{X}\mathbf{X}}]_{i,j} = k(\bfx_i,\bfx_j)$.

\subsection{Kernels on discrete variables}
\vspace{-4pt}
We first review some kernel terminology~\citep{scholkopf2001learning} that is needed in the rest of the paper.
\vspace{-2pt}
\begin{definition}[Gram Matrix] 
    Given a function $k : \calX \times \calX \rightarrow \bbR$ and data $x_1, \cdots, x_n \in \calX$, the $n \times n$ matrix $K$ with elements $[K]_{ij} = k(x_i, x_j)$ is called the Gram matrix of $k$ with respect to $x_1, \cdots, x_n$.
\end{definition}
\vspace{-4pt}
\begin{definition}[Positive Definite Matrix]
    A real $n \times n$ matrix $K$ satisfying $\sum_{i,j} a_i [K]_{ij} a_j \ge 0$ for all $a_i \in \bbR$ is called positive definite~(PD)\footnote{Sometimes, different terms are used, semi-positive definite for $\sum_{i,j} a_i [K]_{ij} a_j \ge 0$ and positive definite for $\sum_{i,j} a_i [K]_{ij} a_j > 0$. Here, we stick to the definition in~\citep{scholkopf2001learning}.}.
\end{definition}
\vspace{-4pt}
\begin{definition}[Positive Definite Kernel]\label{def:pd_kernel}
    A function $k : \calX \times \calX \rightarrow \bbR$ which  gives rise to a positive definite Gram matrix for all $n \in \bbN$ and all $x_1, \cdots, x_n \in \calX$ is called a positive definite~(PD) kernel, or simply a kernel.
\end{definition}
\vspace{-2pt}

A search space which consists of discrete variables, including both nominal and ordinal variables, can be represented as a graph~\citep{kondor2002diffusion,oh2019combinatorial}.
In this graph each vertex represents one state of exponentially many joint states of the discrete variables. The edges represent relations between these states (e.g. if they are similar) ~\citep{oh2019combinatorial}. 
With a graph representing a search space of discrete variables, kernels on a graph can be used for BO.
In~\citep{smola2003kernels}, for a positive decreasing function $f$ and a graph $\calG = (\calV, \calE)$ whose graph Laplacian $L(\calG)$\footnote{In this paper, we use a (unnormalized) graph Laplacian $L(\calG) = D - A$ while, in \citep{smola2003kernels}, symmetric normalized graph Laplacian, $L^{sym}(\calG) = D^{-1/2}(D - A)D^{-1/2}$.~($A$ : adj. mat. / $D$ : deg. mat.) Kernels are defined for both.} has the eigendecomposition $U \Lambda U^T$, it is shown that a kernel can be defined as
\begin{equation}\label{eq:kernel_graph_laplacian}
    k_{disc}(v,v' \vert \beta) = [U f(\Lambda \vert \beta) U^T]_{v,v'}
\end{equation} \par
where $\beta \ge 0$ is a kernel parameter and $f$ is a positive decreasing function. 
It is the reciprocal of a regularization operator~\citep{smola2003kernels} which penalizes high frequency components in the spectrum.

\section{Mixed Variable Bayesian Optimization}
\label{sec:method}

With the goal of obtaining flexible kernels on mixed variables which can model complex dependencies across different types of variables,
we propose the frequency modulated~(FM) kernel. 
Our objective is to enhance the modelling capacity of GP surrogate models and, thereby improve the sample efficiency of mixed-variable BO.
FM kernels use the continuous variables to modulate the frequencies of the kernel of discrete variables defined on the graph. 
As a consequence, FM kernels can model complex dependencies between continuous and discrete variables.
Specifically, let us start with continuous variables of dimension $D_{\calC}$, and discrete variables represented by the graph $\calG = (\calV, \calE)$ whose graph Laplacian $L(\calG)$ has eigendecompostion $U \Lambda U^T$.
To define a frequency modulated kernel we consider the function $k : (\bbR^{D_{\calC}} \times \calV) \times (\bbR^{D_{\calC}} \times \calV) \Rightarrow \bbR$ of the following form
\begin{align} \label{eq:FM_kernel_form}
    &k((\bfc,v),(\bfc',v') \vert \beta, \bftheta) \nonumber \\
    &= \sum_{i=1}^{\vert \calV \vert} [U]_{v,i}f(\lambda_i, \Vert \bfc - \bfc' \Vert_{\bftheta} \vert \beta)[U]_{v',i} 
\end{align}\vspace{-6pt}\\
where $\Vert \bfc - \bfc' \Vert_{\bftheta}^2 = \sum_{d=1}^{D_{\calC}} (c_d - c_d')^2 / \theta_d^2$ and ($\bftheta$, $\beta$) are tunable parameters. $f$ is the frequency modulating function defined below in Def.~\ref{def:fm_function}.

The function $f$ in Eq.~\eqref{eq:FM_kernel_form} takes frequency $\lambda_i$ and distance $\Vert \bfc - \bfc' \Vert_{\bftheta}^2$ as arguments, and its output is combined with the basis $[U]_{v,i}$.
That is, the function $f$ processes the information in each eigencomponent separately while Eq.~\eqref{eq:FM_kernel_form} then sums up the information processed by $f$.
Note that unlike kernel addition and kernel product,\footnote{e.g $k_{add}((\bfc,v),(\bfc',v')) = e^{-\Vert \bfc - \bfc' \Vert_{\bftheta}^2} + k_{disc}(v,v')$ and $k_{prod}((\bfc,v),(\bfc',v')) = e^{-\Vert \bfc - \bfc' \Vert_{\bftheta}^2} \cdot k_{disc}(v,v')$}, the distance $\Vert \bfc - \bfc' \Vert_{\bftheta}^2$ influences each eigencomponent separately as illustrated in Figure.\ref{fig:FM_GM}.
Unfortunately, Eq.~\eqref{eq:FM_kernel_form} with an arbitrary function $f$ does not always define a positive definite kernel.
Moreover, Eq.~\eqref{eq:FM_kernel_form} with an arbitrary function $f$ may return higher kernel values for less similar points, which is not expected from a proper similarity measure~\citep{vert2004primer}.
To this end, we first specify three properties of functions $f$ such that Eq.~\eqref{eq:FM_kernel_form} guaranteed to be a positive definite kernel and a proper similarity measure at the same time.
Then, we motivate the necessity of each of the properties in the following subsections.
\begin{definition}[Frequency modulating function]\label{def:fm_function}
    A frequency modulating function is a function $f : \bbR^+ \times \bbR \rightarrow \bbR$ satisfying the three properties below.
    \vspace{-4pt}
    \begin{enumerate}[leftmargin=34pt, itemsep=-2pt]
        \item [\textbf{FM-P1}] For a fixed $t \in \bbR$, $f(s, t)$ is a positive and decreasing function with respect to $s$ on $[0,\infty)$.
        \item [\textbf{FM-P2}] For a fixed $s \in \bbR^+$, $f(s, \Vert \bfc - \bfc' \Vert_{\bftheta})$ is a positive definite kernel on $(\bfc, \bfc') \in \bbR^{D_{\calC}} \times \bbR^{D_{\calC}}$.
        \item [\textbf{FM-P3}] For $t_1 < t_2$, $h_{t_1, t_2}(s) = f(s, t_1) - f(s, t_2)$ is positive, strictly decreasing and convex w.r.t $s \in \bbR^+$.
    \end{enumerate}
\end{definition}

\begin{definition}[FM kernel]
    A FM kernel is a function on $(\bbR^{D_{\calC}} \times \calV) \times (\bbR^{D_{\calC}} \times \calV)$ of the form in Eq.~\eqref{eq:FM_kernel_form}, where $f$ is a frequency modulating function on $\bbR^+ \times \bbR$.
\end{definition}

\begin{figure}
    \centering
    \begin{minipage}[c]{0.8\columnwidth}
        \vspace{-4pt}
        \centering
        \includegraphics[width=0.8\columnwidth]{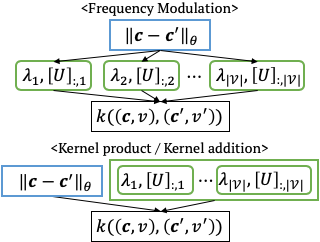}
        \caption{Influence on eigencomponents}\label{fig:FM_GM}
        \vspace{-4pt}
    \end{minipage}
    \vspace{-4pt}
\end{figure}

\subsection{Frequency Regularization of FM kernels}
\vspace{-2pt}

In~\citep{smola2003kernels}, it is shown that Eq.~\eqref{eq:kernel_graph_laplacian} defines a kernel that regularizes the eigenfunctions with high frequencies when $f$ is positive and decreasing.
It is also shown that the reciprocal of $f$ in Eq.~\eqref{eq:kernel_graph_laplacian} is a corresponding regularization operator.
For example, the diffusion kernel defined with $f(\lambda) = \exp(-\beta \lambda)$ corresponds to the regularization operator $r(\lambda) = \exp(\beta \lambda)$.
The regularized Laplacian kernel defined with $f(\lambda) = 1 / (1 + \beta \lambda)$ corresponds to the regularization operator $r(\lambda) = 1 + \beta \lambda$.
Both regularization operators put more penalty on higher frequencies $\lambda$.

Therefore, the property \textbf{FM-P1} forces FM kernels to have the same regularization effect of promoting a smoother function by penalizing the eigenfunctions with high frequencies.

\subsection{Positive Definiteness of FM kernels}\label{subsec:pd_fm}
\vspace{-2pt}

Determining whether Eq.\ref{eq:FM_kernel_form} defines a positive definite kernel is not trivial.
The reason is that the gram matrix $[k((\bfc_i, v_i),(\bfc_j, v_j))]_{i,j}$ is not determined only by the entries $v_i$ and $v_j$, but these entries are additionally affected by different distance terms $\Vert \bfc_i - \bfc_j \Vert_{\bftheta}$.
To show that FM kernels are positive definite, it is sufficient to show that $f(\lambda_i, \Vert \bfc - \bfc' \Vert_{\bftheta} \vert ~\beta)$ is positive definite on $(\bfc, \bfc') \in \bbR^{D_{\calC}} \times \bbR^{D_{\calC}}$.
\begin{theorem}\label{thm:positive_definite}
    If $f(\lambda, \Vert \bfc - \bfc' \Vert_{\bftheta} \vert ~\beta)$ defines a positive definite kernel with respect to $\bfc$ and $\bfc'$, then the FM kernel with such $f$ is positive definite \textbf{jointly} on $\bfc$ and $v$.
    That is, the positive definiteness of $f(\lambda, \Vert \bfc - \bfc' \Vert_{\bftheta} \vert ~\beta)$ on $\bbR^{D_{\calC}}$ implies the positive definiteness of the FM kernel on $\bbR^{D_{\calC}} \times \calV$.
\end{theorem}
\vspace{-4pt}
\begin{proof}
    See Supp.~Sec.\ref{supp:sec:thm_definite}, Thm.~\ref{supp:thm:positive_definite}.
\end{proof}
Note that Theorem \ref{thm:positive_definite} shows that the property \textbf{FM-P2} guarantees that FMs kernels are positive definite jointly on $\bfc$ and $v$.

In the current form of Theorem~\ref{thm:positive_definite}, the frequency modulating functions depend on the distance $\Vert \bfc - \bfc' \Vert_{\bftheta}$.
However, the proof does not change for the more general form of $f(\lambda, \bfc, \bfc' \vert \alpha, \beta)$, where $f$ does not depend on $\Vert \bfc - \bfc' \Vert_{\bftheta}$.
Hence, Theorem~\ref{thm:positive_definite} can be extended to the more general case that $f(\lambda, \bfc, \bfc' \vert \alpha, \beta)$ is positive definite on $(\bfc, \bfc') \in \bbR^{D_{\calC}} \times \bbR^{D_{\calC}}$.

\subsection{Frequency Modulation Principle}\label{subsec:principle}
\vspace{-2pt}

A kernel, as a similarity measure, is expected to return higher values for pairs of more similar points and vice versa~\citep{vert2004primer}.
We call such behavior the \textit{similarity measure behavior}.

In Eq.~\eqref{eq:FM_kernel_form}, the distance $\Vert \bfc - \bfc' \Vert_{\bftheta}$ represents a quantity in the ``spatial'' domain interacting with quantities $\lambda_i$s in the ``frequency'' domain.
Due to the interplay between the two different domains, the kernels of the form Eq.~\eqref{eq:FM_kernel_form} do not exhibit the similarity measure behavior \emph{for an arbitrary function} $f$.
Next, we derive a sufficient condition on $f$ for the similarity measure behavior to hold for FM kernels.

Formally, the similarity measure behavior is stated as 
\vspace{-2pt}
\begin{align}\label{eq:FM_principle_distance}
    &\Vert \bfc - \bfc' \Vert_{\bftheta} \le \Vert \Tilde{\bfc} - \Tilde{\bfc}' \Vert_{\bftheta} \nonumber \\
    &~~\Rightarrow~~
    k((\bfc,v),(\bfc',v')) \ge k((\Tilde{\bfc},v),(\Tilde{\bfc}',v'))
\end{align}
\vspace{-4pt}\\
or equivalently,
\begin{align}\label{eq:similarity_measure}
    &\Vert \bfc - \bfc' \Vert_{\bftheta} \le \Vert \Tilde{\bfc} - \Tilde{\bfc}' \Vert_{\bftheta} \nonumber \\
    &~~\Rightarrow~~
    \sum_{i=1}^{\vert \calV \vert} 
    [U]_{v,i}
    h_{t_1, t_2}(\lambda_i \vert \beta)
    [U]_{v',i} \ge 0
\end{align}
\vspace{-4pt}\\
where $h_{t_1, t_2}(\lambda \vert \beta) = f(\lambda, t_1 \vert \beta) - f(\lambda, t_2 \vert \beta)$, $t_1 = \Vert \bfc - \bfc' \Vert_{\bftheta}$ and $t_2 = \Vert \Tilde{\bfc} - \Tilde{\bfc}' \Vert_{\bftheta}$.

\begin{theorem}\label{thm:nonnegative_valued}
    For a connected and weighted undirected graph $\calG=(\calV,\calE)$ with non-negative weights on edges, define a similarity (or kernel) $a(v,v') = [Uh(\Lambda)U^T]_{v,v'}$, where $U$ and $\Lambda$ are eigenvectors and eigenvalues of the graph Laplacian $L(\calG)=U \Lambda U^T$. 
    If $h$ is any non-negative and strictly decreasing convex function on $[0, \infty)$, then $a(v,v') \ge 0$ for all $v, v' \in \calV$.
\end{theorem}
Therefore, these conditions on $h(\Lambda)$ result in a similarity measure $a$ with only positive entries, which in turn proves property Eq.~\eqref{eq:similarity_measure}. 
Here, we provide a proof of the theorem for a simpler case with an unweighted complete graph, where Eq.~\eqref{eq:similarity_measure} holds without the convexity condition on $h$. 
\begin{proof}
    For a unweighted complete graph with $n$ vertices, we have  eigenvalues $\lambda_1 = 0$, $\lambda_2 = \cdots = \lambda_n = n$ and eigenvectors such that $[U]_{\cdot1}=1/\sqrt{n}$ and $\sum_{i=1}^n [U]_{v,i}[U]_{v',i} = \delta_{vv'}$.
    For $v \neq v'$, the conclusion in Eq.~\eqref{eq:similarity_measure}, $\sum_{i=1}^n h(\lambda_i)[U]_{v,i}[U]_{v',i}$ becomes $h(0)/n + h(n)\sum_{i=2}^n [U]_{v,i}[U]_{v',i} = (h(0) - h(n)) / n$ in which non-negativity follows with decreasing $h$.
    
    For the complete proof, see
    Thm.~\ref{supp:thm:nonnegative_valued} in Supp.~Sec.~\ref{supp:sec:thm_positivity}.
\end{proof}

Theorem~\ref{thm:nonnegative_valued} thus shows that the property \textbf{FM-P3} is sufficient for Eq.~\eqref{eq:similarity_measure} to hold.
We call the property \textbf{FM-P3} the \textit{frequency modulation principle}.
Theorem \ref{thm:nonnegative_valued} also implies the non-negativity of many kernels derived from graph Laplacian.

\begin{corollary}
    The random walk kernel derived from the symmetric normalized Laplacian~\citep{smola2003kernels}, the diffusion kernels~\citep{kondor2002diffusion,oh2019combinatorial} and the regularized Laplacian kernel~\citep{smola2003kernels} derived from symmetric normalized or unnormalized Laplacian, are all non-negatived valued.
\end{corollary}
\begin{proof}
    See Cor.~\ref{supp:coro:nonnegative_valued} in Supp.~Sec.~\ref{supp:sec:thm_positivity}.
\end{proof}

\subsection{FM kernels in practice}\label{subsec:FM_in_practice}
\vspace{-2pt}

\paragraph{Scalability} Since the (graph Fourier) frequencies and basis functions are computed by the eigendecomposition of cubic computational complexity, a plain application of frequency modulation makes the computation of FM kernels prohibitive for a large number of discrete variables.
Given $P$ discrete variables where each variable can be individually represented by a graph $\calG_p$, the discrete part of the search space can be represented as a product space, $\calV = \calV_1 \times \cdots \times \calV_P$.

In this case, we define FM kernels on $\bbR^{D_{\calC}} \times \calV = \bbR^{D_{\calC}} \times (\calV_1 \times \cdots \times \calV_P)$ as
\begin{align}\label{eq:scalable_FM}
    &k((\bfc,\bfv),(\bfc',\bfv') \vert \bfalpha, \bfbeta, \bftheta) = \prod_{p=1}^P k_p((\bfc, v_p),(\bfc', v_p') \vert \beta_p, \bftheta) \nonumber \\
    &= \prod_{p=1}^P \sum_{i=1}^{\vert \calV_p \vert} [U^p]_{v_p,i}f(\lambda_i^p, \alpha_p \Vert \bfc - \bfc' \Vert_{\bftheta} \vert \beta_p)[U^p]_{v_p',i}
\end{align} \par
where $\bfv = (v_1, \cdots, v_P$, $\bfv' = (v_1', \cdots, v_P'$, $\bfalpha = (\alpha_1, \cdots, \alpha_P)$ $\bfbeta = (\beta_1, \cdots, \beta_P)$ and the graph Laplacian is given as $L(\calG_p)$ with the eigendecomposition $U_p \text{diag}[\lambda_1^p, \cdots, \lambda_{\Vert \calV_p \Vert}^p]U_p^T$.

Eq.\ref{eq:scalable_FM} should not be confused with the kernel product of kernels on each $\calV_p$.
Note that the distance $\Vert \bfc - \bfc' \Vert_{\bftheta}$ is shared, which introduces the coupling among discrete variables and thus allows more modeling freedom than a product kernel.
In addition to the coupling, the kernel parameter $\alpha_p$s lets us individually determine the strength of the frequency modulation.


\paragraph{Examples} Defining a FM kernel amounts to constructing a frequency modulating function.
We introduce examples of flexible families of frequency modulating functions.

\begin{proposition}\label{prop:FM_family}
    For $S \in (0, \infty)$, a finite measure $\mu$ on $[0, S]$, $\mu$-measurable $\tau : [0, S] \Rightarrow [0, 2]$ and $\mu$-measurable $\rho : [0, S] \Rightarrow \bbN$, the function of the form below is a frequency modulating function.
    \vspace{-2pt}
    \begin{align}\label{eq:FM_family}
        &f(\lambda, \alpha \Vert \bfc - \bfc' \Vert_{\bftheta} \vert \beta) \nonumber \\
        &= \int_0^S \frac{1}{(1 + \beta \lambda + \alpha \Vert \bfc - \bfc' \Vert_{\bftheta}^{\tau(s)})^{\rho(s)}} \mu(ds)
    \end{align}
    \vspace{-4pt}
\end{proposition}
\vspace{-4pt}
\begin{proof}
    See Supp.~Sec.\ref{supp:sec:FM_examples}, Prop.\ref{supp:prop:FM_family}. 
\end{proof}

Assuming $S=1$ and $\tau(s) = 2$, Prop.~\ref{prop:FM_family} gives $(1 + \beta \lambda + \alpha \Vert \bfc - \bfc' \Vert_{\bftheta}^2)^{-1}$ with $\rho(s) = 1$ and $\mu(ds) = ds$, and $\sum_{n=1}^N a_n(1 + \beta \lambda + \alpha \Vert \bfc - \bfc' \Vert_{\bftheta}^2)^{-n}$ with $\rho(s) = \lfloor Ns \rfloor$ and $\mu(\{ n/N \}) = a_n \ge 0$ and $\mu([\{n/N\}_{n=1,\cdots,N}^c) = 0$.

\subsection{Extension of the Frequency Modulation}

Frequency modulation is not restricted to distances on Euclidean spaces but it is applicable to any arbitrary space with a kernel defined on it.
As a concrete example of frequency modulation by kernels, we show a non-stationary extension where $f$ does not depend on $\Vert \bfc - \bfc' \Vert_{\bftheta}$ but on the neural network kernel $k_{NN}$~\citep{rasmussen2003gaussian}.
Consider Eq.~\eqref{eq:FM_kernel_form} with $f = f_{NN}$ as follows.
\begin{equation}\label{eq:FM_extension}
    f_{NN}(\lambda, k_{NN}(\bfc, \bfc' \vert \Sigma) \vert \beta) = \frac{1}{2 + \beta \lambda - k_{NN}(\bfc, \bfc' \vert \Sigma)}
\end{equation} \par
where $k_{NN}(\bfc, \bfc' \vert \Sigma) = \frac{2}{\pi} \arcsin{\Big( \frac{2 \bfc^T \Sigma \bfc'}{(1 + \bfc^T \Sigma \bfc)(1 + \bfc'^T \Sigma \bfc')} \Big)}$ is the neural network kernel~\citep{rasmussen2003gaussian}.

Since the range of $k_{NN}$ is $[-1, 1]$, $f_{NN}$ is positive and thus satisfies \textbf{FM-P1}.
Through Eq.\ref{eq:FM_extension}, Eq.\ref{eq:FM_kernel_form} is positive definite~(Supp. Sec.\ref{supp:sec:FM_examples}, Prop.\ref{supp:prop:fm_kernel_extension_PD}) and thus property \textbf{FM-P2} is satisfied.
If the premise $t_1 < t_2$ of the property \textbf{FM-P3} is replaced by $t_1 > t_2$, then \textbf{FM-P3} is also satisfied.
In contrast to the frequency modulation principle with distances in Eq.~\eqref{eq:FM_principle_distance}, the frequency modulation principle with a kernel is formalized as 
\begin{align}\label{eq:FM_principle_innerprod}
    &k_{NN}(\bfc, \bfc' \vert \Sigma) \ge k_{NN}(\Tilde{\bfc}, \Tilde{\bfc}' \vert \Sigma) \nonumber \\
    &\Rightarrow k((\bfc,v),(\bfc',v')) \ge k((\Tilde{\bfc},v),(\Tilde{\bfc}',v'))
\end{align} \par
Note that $k_{NN}(\bfc, \bfc' \vert \Sigma)$ is a similarity measure and thus the inequality is not reversed unlike Eq.~\eqref{eq:FM_principle_distance}.

All above arguments on the extension of the frequency modulation using a nonstationary kernel hold also when the $k_{NN}$ is replaced by an arbitrary positive definite kernel.
The only required condition is that a kernel has to be upper bounded, \emph{i.e.}, $k_{NN}(\bfc, \bfc') \leq C$, needed for \textbf{FM-P1} and \textbf{FM-P2}.

\section{Related work}
\label{sec:related_work}
\vspace{-4pt}
\vspace{-4pt}

On continuous variables, many sophisticated kernels have been proposed~\citep{wilson2015kernel,samo2015generalized,remes2017non,oh2018bock}.
In contrast, kernels on discrete variables have been studied less~\citep{haussler1999convolution,kondor2002diffusion,smola2003kernels}.
To our best knowledge, most of existing kernels on mixed variables are constructed by a kernel product~\cite{swersky2013multi,li2016contextual} with some exceptions~\citep{krause2011contextual,swersky2013multi,fiducioso2019safe}, which rely on kernel addition.


In mixed variable BO, non-GP surrogate models are more prevalent, including SMAC~\citep{hutter2011sequential} using random forest and TPE~\citep{bergstra2011algorithms} using a tree structured density estimator.
Recently, by extending the approach of using Bayesian linear regression for discrete variables~\citep{baptista2018bayesian}, \citet{daxberger2019mixed} proposes Bayesian linear regression with manually chosen basis functions on mixed variables, providing a regret analysis using Thompson sampling as an acquisition function.
Another family of approaches utilizes a bandit framework to handle the acquisition function optimization on mixed variables with theoretical analysis~\citep{gopakumar2018algorithmic,nguyen2019bayesian,ru2020bayesian}.
\citet{nguyen2019bayesian} use GP in combination with multi-armed bandit to model category-specific continuous variables and provide regret analysis using GP-UCB.
Among these approaches, \citet{ru2020bayesian} also utilize information across different categorical values, which --in combination with the bandit framework-- makes itself the most competitive method in the family.

Our focus is to extend the modelling prowess and flexibility of pure GPs for surrogate models on problems with mixed variables.
We propose frequency modulated kernels, which are kernels that are specifically designed to model the complex interactions between continuous and discrete variables.

In architecture search, approaches using weight sharing such as DARTS~\citep{liu2018darts} and ENAS~\citep{pham2018efficient} are gaining popularity.
In spite of their efficiency, methods training neural networks from scratch for given architectures outperform approaches based on weight sharing~\citep{dong2021nats}.
Moreover, the joint optimization of learning hyperparameters and architectures is under-explored with a few exceptions such as BOHB~\citep{falkner2018bohb} and autoHAS~\citep{dong2020autohas}.
Our approach proposes a competitive option to this challenging optimization of mixed variable functions with expensive evaluation cost.


\vspace{-4pt}

\section{Experiments}
\label{sec:experiment}
\vspace{-4pt}
\vspace{-4pt}

\begin{figure*}[!t]
    \vspace{-6pt}
    \centering
    \includegraphics[width=\linewidth]{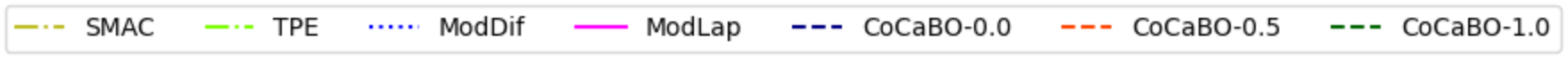}\\
    \vspace{-3pt}
    \minipage{0.33\textwidth}
    	\centering
        \includegraphics[width=\linewidth]{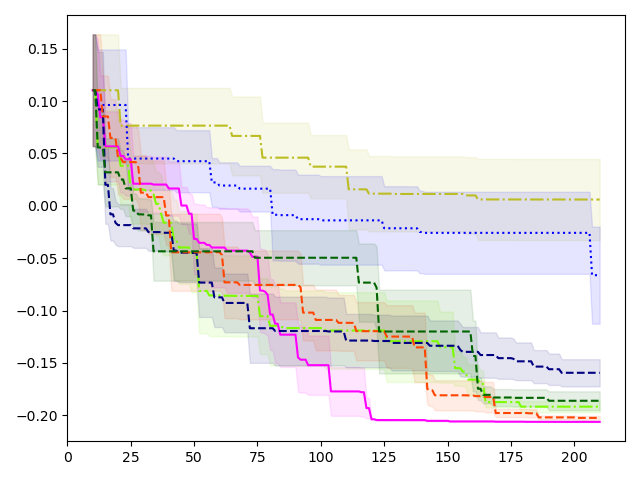}
    \endminipage\hfill
    \minipage{0.33\textwidth}
        \centering
        \includegraphics[width=\linewidth]{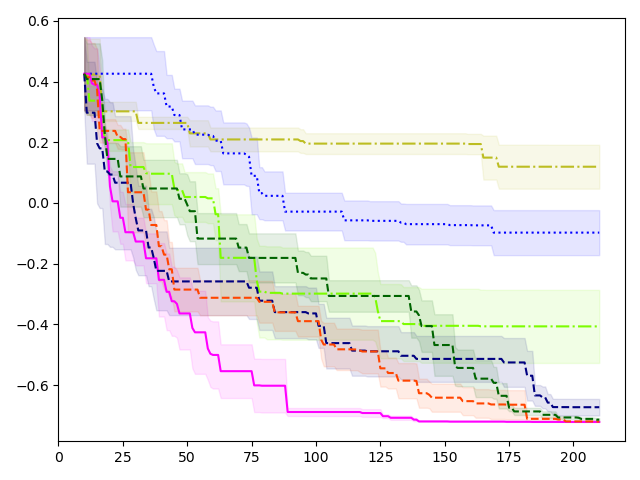}
    \endminipage\hfill
    \minipage{0.33\textwidth}%
        \centering
        \includegraphics[width=\linewidth]{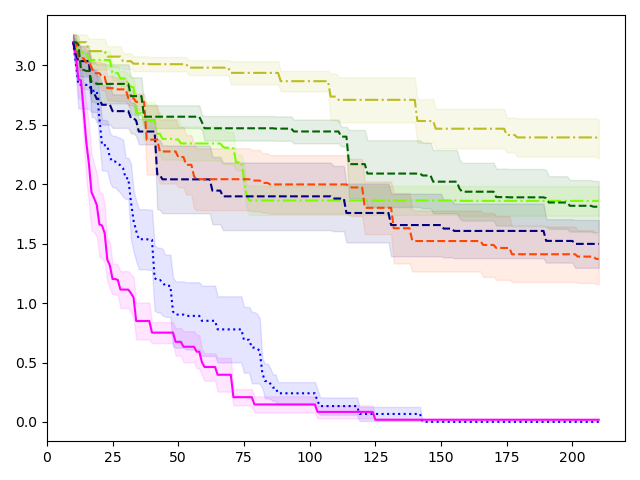}
    \endminipage\hfill
    
    \begin{footnotesize}
        \renewcommand{\arraystretch}{1.2}
        \setlength{\tabcolsep}{2pt}
        \begin{tabular}{c|c|c|c|c|c|c|c}\hline
                 & SMAC & TPE & ModDif & ModLap & CoCaBO-0.0 & CoCaBO-0.5 & CoCaBO-1.0 \\ \hline
        Func2C   & $+0.006\pm0.039$ & $-0.192\pm0.005$ & $-0.066\pm0.046$ & $\mathbf{-0.206\pm0.000}$ & $-0.159\pm0.013$ & $\mathit{-0.202\pm0.002}$ & $-0.186\pm0.009$ \\
        Func3C   & $+0.119\pm0.072$ & $-0.407\pm0.120$ & $-0.098\pm0.074$ & $\mathbf{-0.722\pm0.000}$ & $-0.673\pm0.027$ & $\mathit{-0.720\pm0.002}$ & $-0.714\pm0.005$ \\
        Ackley5C & $+2.381\pm0.165$ & $+1.860\pm0.125$ & $\mathbf{+0.001\pm0.000}$ & $\mathit{+0.019\pm0.006}$ & $+1.499\pm0.201$ & $+1.372\pm0.211$ & $+1.811\pm0.217$ \\ \hline
        \end{tabular}
    \end{footnotesize}
    \vspace{-8pt}
    \caption{Func2C(left), Func3C(middle), Ackley5C(right)~(Mean$\pm$Std.Err. of 5 runs)}\label{fig:exp_synthetic}
    \vspace{-16pt}
\end{figure*}

To demonstrate the improved sample efficiency of GP BO using FM kernels~(BO-FM) we study various mixed variable black-box function optimization tasks, including 3 synthetic problems from~\citet{ru2020bayesian}, 2 hyperparameter optimization problems~(SVM~\citep{smola2003kernels} and XGBoost~\citep{chen2016xgboost}) and the joint optimization of neural architecture and SGD hyperparameters.

As per our method, we consider \textsc{ModLap} which is of the form Eq.~\ref{eq:scalable_FM} with the following frequency modulating function.
\vspace{-10pt}
\begin{equation}\label{eq:fmf_lap}
    f_{Lap}(\lambda, \Vert \bfc - \bfc' \Vert_{\bftheta} \vert \alpha, \beta) = \frac{1}{1 + \beta \lambda + \alpha \Vert \bfc - \bfc' \Vert_{\bftheta}^2}
\end{equation} \vspace{-10pt}\\
Moreover, to empirically demonstrate the importance of the similarity measure behavior, we consider another kernel following the form of Eq.~\ref{eq:scalable_FM} but disrespecting the frequency modulation principle with the function
\vspace{-6pt}
\begin{align}\label{eq:fmf_dif}
    f_{Dif}(\lambda, \Vert \bfc - \bfc' \Vert_{\bftheta} \vert \alpha, \beta) = \exp{(-( 1 + \alpha \Vert \bfc - \bfc' \Vert_{\bftheta}^2) \beta \lambda)}
\end{align} \vspace{-14pt}\\
We call the kernel constructed with this function \textsc{ModDif}.
The implementation of these kernels is publicly available.\footnote{\url{https://github.com/ChangYong-Oh/FrequencyModulatedKernelBO}}

In each round, after updating with an evaluation, we fit a GP surrogate model using marginal likelihood maximization with 10 random initialization until convergence~\citep{rasmussen2003gaussian}.
We use the expected improvement~(EI) acquisition function~\citep{donald1998efficient} and optimize it by repeated alternation of L-BFGS-B~\citep{zhu1997algorithm} and hill climbing~\citep{skiena1998algorithm} until convergence.
More details on the experiments are provided in Supp.~Sec.~\ref{supp:sec:exp_details}.

\vspace{-4pt}
\paragraph{Baselines}
For synthetic problems and hyperparameter optimization problems below, baselines we consider\footnote{The methods~\citep{daxberger2019mixed,nguyen2019bayesian} whose code has not been released are excluded.} are SMAC\footnote{\url{https://github.com/automl/SMAC3}}~\citep{hutter2011sequential}, TPE\footnote{\url{http://hyperopt.github.io/hyperopt/}}~\citep{bergstra2011algorithms}, and CoCaBO\footnote{\url{https://github.com/rubinxin/CoCaBO_code}}~\citep{ru2020bayesian} which consistently outperforms One-hot BO~\citep{gpyopt2016} and EXP3BO~\citep{gopakumar2018algorithmic}.
For CoCaBO, we consider 3 variants using different mixture weights.\footnote{Learning the mixture weight is not supported in the implementation, we did not include it. Moreover, as shown in~\citet{ru2020bayesian}, at least one of 3 variants usually performs better than learning the mixture weight.}

\subsection{Synthetic problems}
\vspace{-4pt}

We test on 3 synthetic problems proposed in~\citet{ru2020bayesian}\footnote{In the implementation provided by the authors, only Func2C and Func3C are supported. We implemented Ackley5C.}.
Each of the synthetic problems has the search space as in Tab.~\ref{tab:synthetic_search_space}.
Details of synthetic problems can be found in~\citet{ru2020bayesian}.
\vspace{-6pt}
\begin{table}[!ht]
    \centering
    \begin{tabular}{c|c|c} \hline
        & Conti. Space & Num. of Cats. \\ \hline
        Func2C & $[-1, 1]^2$ & 3, 5\\
        Func3C & $[-1, 1]^2$ & 3, 5,4\\
        Ackley5C & $[-1, 1]$ & 17, 17, 17, 17, 17 \\ \hline
    \end{tabular}
    \vspace{-6pt}
    \caption{Synthetic Problem Search Spaces}
    \label{tab:synthetic_search_space}
    \vspace{-14pt}
\end{table}

On all 3 synthetic benchmarks, \textsc{ModLap} shows competitive performance~(Fig.~\ref{fig:exp_synthetic}).
On Func2C and Func3C, \textsc{ModLap} performs the best, while on Ackley5C \textsc{ModLap} is at the second place, marginally further from the first.
Notably, even on Func2C and Func3C, where \textsc{ModDif} underperforms significantly, \textsc{ModLap} exhibits its competitiveness, which empirically supports that the similarity measure behavior plays an important role in the surrogate modeling in Bayesian optimization.
Note that TPE and CoCaBO have much shorter wall-clock runtime.

\subsection{Hyperparameter optimization problems}

\begin{figure*}[!th]
    \vspace{-4pt}
    \centering
    \includegraphics[width=\linewidth]{figures/main-legend.png}\\
    \vspace{-3pt}
    \minipage{0.33\textwidth}
    	\centering
        \includegraphics[width=\linewidth]{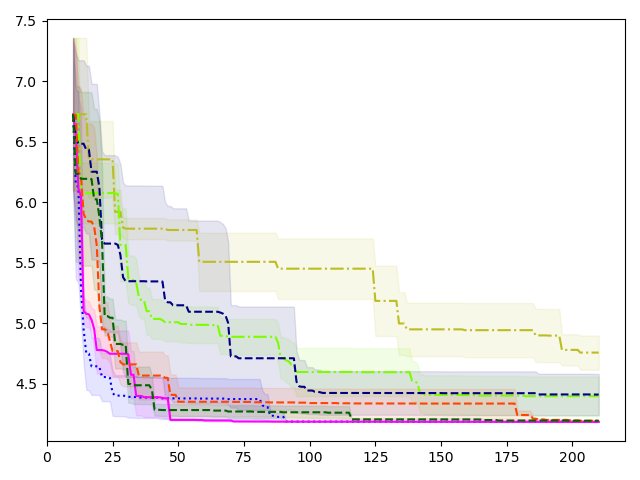}
    \endminipage\hfill
    \minipage{0.33\textwidth}
        \renewcommand{\arraystretch}{1.2}
        \centering
        \footnotesize
        \setlength{\tabcolsep}{2pt}
        \begin{tabular}{c|c|c} \hline
        SVM & Method & XGBoost \\ \hline
        $4.759\pm.141$ & SMAC       & $.1215\pm .0045$ \\
        $4.399\pm.163$ & TPE        & $.1084\pm .0007$\\
        $\mathit{4.188\pm.001}$ & ModDif     & $\mathit{.1071\pm .0013}$\\
        $\mathbf{4.186\pm.002}$ & ModLap     & $\mathbf{.1038\pm .0003}$\\
        $4.412\pm.170$ & CoCaBO-0.0 & $.1184\pm .0062$\\
        $4.196\pm.004$ & CoCaBO-0.5 & $.1079\pm .0010$\\
        $4.196\pm.004$ & CoCaBO-1.0 & $.1086\pm .0008$
        \end{tabular}
    \endminipage\hfill
    \minipage{0.33\textwidth}%
        \centering
        \includegraphics[width=\linewidth]{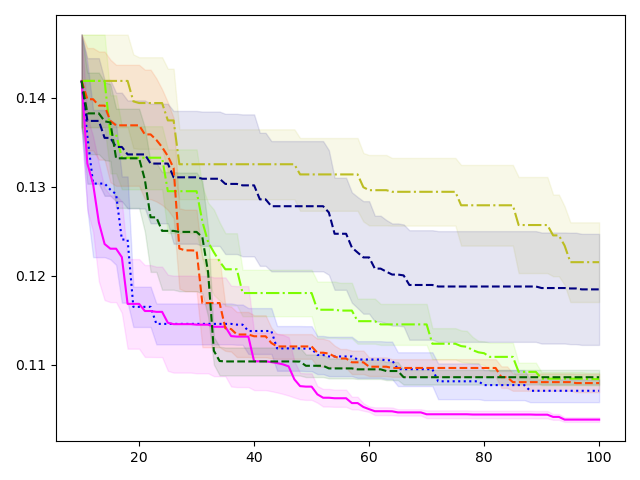}
    \endminipage\hfill
    \vspace{-8pt}
    \caption{SVM(left), XGBoost(right)~(Mean$\pm$Std.Err. of 5 runs)}\label{fig:exp_hyper}
    \vspace{-12pt}
\end{figure*}

Now we consider a practical application of Bayesian optimization over mixed variables.
We take two machine learning algorithms, SVM~\citep{smola2003kernels} and XGBoost~\citep{chen2016xgboost} and optimize their hyperparameters.

\paragraph{SVM} We optimize hyperparameters of NuSVR in scikit-learn~\citep{pedregosa2011scikit}.
We consider 3 categorical hyperparameters and 3 continuous hyperparameters~(Tab.~\ref{tab:svm_hyper}) and for continuous hyperparameters we search over $\log_{10}$ transformed space of the range.
\vspace{-6pt}
\begin{table}[!ht]
    \centering
    \begin{tabular}{c|c} \hline
        NuSVR param.\tablefootnote{\url{https://scikit-learn.org/stable/modules/generated/sklearn.svm.NuSVR.html}} & Range \\ \hline
        kernel & $\{$linear, poly, RBF, sigmoid $\}$  \\
        gamma & $\{$scale, auto $\}$ \\
        shrinking & $\{$on, off $\}$ \\ \hline
        C & $[10^{-4}, 10]$ \\
        tol & $[10^{-6}, 1]$ \\
        nu & $[10^{-6}, 1]$ \\ \hline
        \end{tabular}
    \vspace{-4pt}
    \caption{NuSVR hyperparameters}\label{tab:svm_hyper}
    \vspace{-8pt}
\end{table}\\
For each of 5 split of Boston housing dataset with train:test(7:3) ratio, NuSVR is fitted on the train set and RMSE on the test set is computed.
The average of 5 test RMSE is the objective.

\paragraph{XGBoost} We consider 1 ordinal, 3 categorical and 4 continuous hyperparameters~(Tab.~\ref{tab:xgboost_hyper}).
\vspace{-6pt}
\begin{table}[!ht]
    \centering
    \begin{tabular}{c|c} \hline
        XGBoost param.\tablefootnote{\url{https://xgboost.readthedocs.io/en/latest/parameter.html}} & Range \\ \hline
        max\_depth & $\{1,\cdots,10\}$  \\ \hline
        booster & $\{$gbtree, dart$\}$ \\
        grow\_policy & $\{$depthwise, lossguide$\}$ \\
        objective & $\{$multi:softmax, multi:softprob$\}$ \\ \hline
        eta & $[10^{-6}, 1]$ \\
        gamma & $[10^{-4}, 10]$ \\
        subsample & $[10^{-3}, 1]$ \\
        lambda & $[0, 5]$ \\ \hline
    \end{tabular}
    \vspace{-4pt}
    \caption{XGBoost hyperparameters}\label{tab:xgboost_hyper}
    \vspace{-8pt}
\end{table}\\
For 3 continuous hyperparameters, eta, gamma and subsample, we search over the $\log_{10}$ transformed space of the range.
With a stratified train:test(7:3) split, the model is trained with 50 rounds and the best test error over 50 rounds is the objective of SVM hyperparameter optimization.

In Fig.~\ref{fig:exp_hyper}, \textsc{ModLap} performs the best.
On XGBoost hyperparameter optimization, \textsc{ModLap} exhibits clear benefit compared to the baselines.
Here, \textsc{ModDif} wins the second place in both problems.

\paragraph{Comparison to different kernel combinations}
In Supp.~Sec.~\ref{supp:sec:exp_results}, we also report the comparison with different kernel combinations on all 3 synthetic problems and 2 hyperparameter parameter optimization problems.
We make two observations.
First, \textsc{ModDif}, which does not respect the similarity measure behavior, sometimes severely degrades BO performance.
Second, \textsc{ModLap} obtains equally good final results and consistently finds the better solutions faster than the kernel product. 
This can be clearly shown by comparing the area above the mean curve of BO runs using different kernels.
The area above the mean curve of BO using \textsc{ModLap} is larger than the are above the mean curve of BO using the kernel product.
Moreover, the gap between the area from \textsc{ModLap} and the area from kernel product increases in problems with larger search spaces. 
%
%
Even on the smallest search space, Func2C, \textsc{ModLap} lags behind the kernel product up to around 90th evaluation and outperforms after it.
The benefit of \textsc{ModLap} modeling complex dependency among mixed variables is more prominent in higher dimension problems.

\paragraph{Ablation study on regression tasks}
In addition to the results on BO experiments, we compare FM kernels with kernel addiition and kernel product on three regression tasks from UCI datasets~(Supp.~Sec.~\ref{supp:sec:ablation_study}).
In terms of negative log-likelihood~(NLL), which takes into account uncertainty, ModLap performs the best in two out of three tasks.
Even on the task which is conjectured to have a structure suitable to kernel product, ModLap shows competitive performance.
Moreover, on regression tasks, the importance of the frequency modulation principle is further reinforced.
For full NLL and RMSE comparison and detailed discussion, see Supp.~Sec.~\ref{supp:sec:ablation_study}

\begin{figure*}[!ht]
    \vspace{-6pt}
    \centering
    \minipage{0.5\textwidth}
    	\centering
        \includegraphics[width=\linewidth]{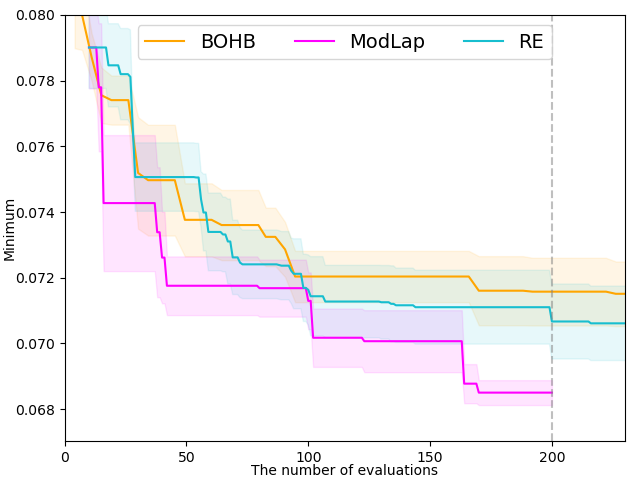}
    \endminipage\hfill
    \minipage{0.5\textwidth}
        \renewcommand{\arraystretch}{1.3}
        \begin{center}
            \setlength{\tabcolsep}{4pt}
            \begin{tabular}{|c|c|c|}
                \hline
                Method & \#Eval. & Mean$\pm$Std.Err. \\
                \hline
                BOHB & 200 & \num{7.158E-02}$\pm$\num{1.0303E-03} \\
                BOHB & 230 & \num{7.151E-02}$\pm$\num{9.8367E-04} \\ 
                BOHB & 600 & \num{6.941E-02}$\pm$\num{4.4320E-04} \\ \hline
                RE   & 200 & \num{7.067E-02}$\pm$\num{1.1417E-03} \\
                RE   & 230 & \num{7.061E-02}$\pm$\num{1.1329E-03} \\
                RE   & 400 & \num{6.929E-02}$\pm$\num{6.4804E-04} \\
                RE   & 600 & \num{6.879E-02}$\pm$\num{1.0039E-03} \\ \hline
                \textsc{ModLap} & 200 & \num[math-rm=\mathbf]{6.850E-02}$\pm$\num[math-rm=\mathbf]{3.7914E-04} \\ \hline
            \end{tabular}\\
            \vspace{2pt}
            \footnotesize For the figure with all numbers above, see Supp.~Sec.~\ref{supp:sec:exp_results}.
        \end{center}
    \endminipage\hfill
    \vspace{-6pt}
    \caption{Joint optimization of the architecture and SGD hyperparameters~(Mean$\pm$Std.Err. of 4 runs)}\label{fig:exp_joint_nas}
    \vspace{-12pt}
\end{figure*}

\subsection{Joint optimization of neural architecture and SGD hyperparameters}

Next, we experiment with BO on mixed variables by optimizing continuous and discrete hyperparameters of neural networks.
The space of discrete hyperparameters $\mathcal{A}$ is modified from the NASNet search space~\citep{zoph2016neural}, which consists of 8,153,726,976 choices.
%
%
The space of continuous hyperparameters $\mathcal{H}$ comprises 6 continuous hyperparameters of the SGD with a learning rate scheduler: learning rate, momentum, weight decay, learning rate reduction factor, 1st reduction point ratio and 2nd reduction point ratio. 
A good neural architecture should both achieve low errors and be computationally modest.
Thus, we optimize the objective $f(a,h) = err_{valid}(a,h) + 0.02 \times FLOP(a)/\max_{a' \in \mathcal{A}}FLOP(a')$.
To increase the separability among smaller values, we  use $\log f(a,h)$ transformed values whenever model fitting is performed on evaluation data.
The reported results are still the original non-transformed $f(a,h)$.

We compare with two strong baselines.
One is BOHB~\citep{falkner2018bohb} which is an evaluation-cost-aware algorithm augmenting unstructured bandit approach~\citep{li2017hyperband} with model-based guidance.
Another is RE~\citep{real2019regularized} based on a genetic algorithm with a novel population selection strategy.
In~\citet{dong2021nats}, on discrete-only spaces, these two outperform competitors including weight sharing approaches such as DARTS~\citep{liu2018darts}, SETN~\citep{dong2019one}, ENAS~\citep{pham2018efficient} and etc.
In the experiment, for BOHB, we use the public implementation\footnote{\url{https://github.com/automl/HpBandSter}} and for RE, we use our own implementation.

For a given set of hyperparameters, with \textsc{ModLap} or RE, the neural network is trained on FashionMNIST for 25 epochs while BOHB adaptively chooses the number of epochs.
For further details on the setup and the baselines we refer the reader to Supp.~Sec.~\ref{supp:sec:exp_details} and~\ref{supp:sec:exp_results}.

We present the results in Fig.~\ref{fig:exp_joint_nas}.
Since BOHB adaptively chooses the budget~(the number of epochs), BOHB is plotted according to the budget consumption.
For example, the y-axis value of BOHB on 100-th evaluation is the result of BOHB having consumed 2,500 epochs~(25 epochs $\times$ 100).

We observe that \textsc{ModLap} finds the best architecture in terms of accuracy and computational cost.
What is more, we observe that \textsc{ModLap} reaches the better solutions faster in terms of numbers of evaluations.
Even though the time to evaluate a new hyperparameter is dominant, the time to suggest a new hyperparameter in \textsc{ModLap} is not negligible in this case.
Therefore, we also provide the comparison with respect the wall-clock time.
It is estimated that RE and BOHB evaluate 230 hyperparameters while \textsc{ModLap} evaluate 200 hyperparameters~(Supp.~Sec.~\ref{supp:sec:exp_details}).
For the same estimated wall-clock time, \textsc{ModLap}(200) outperforms competitors(RE(230), BOHB(230)).

In order to see how beneficial the sample efficiency of BO-FM is in comparison to the baselines, we perform a stress test in which more evaluations are allowed for RE and BOHB.
We leave RE and BOHB for 600 evaluations.
Notably, \textsc{ModLap} with 200 evaluations outperforms both competitors with 600 evaluations~(Fig.~\ref{fig:exp_joint_nas} and Supp.Sec.~\ref{supp:sec:exp_results}).
We conclude that \textsc{ModLap} exhibits higher sample efficiency than the baselines.

\section{Conclusion}

We propose FM kernels to improve the sample efficiency of mixed variable Bayesian optimization.

On the theoretical side, we provide and prove conditions for FM kernels to be positive definite and to satisfy the similarity measure behavior.
Both conditions are not trivial due to the interactions between quantities on two disparate domains, the spatial domain and the frequency domain.

On the empirical side, we validate the effect of the conditions for FM kernels on multiple synthetic problems and realistic hyperparameter optimization problems.
Further, we successfully demonstrate the benefits of FM kernels compared to non-GP based Bayesian Optimization on a challenging joint optimization of neural architectures and SGD hyperparameters.
BO-FM outperforms its competitors, including Regularized evolution, which requires three times as many evaluations. 

We conclude that an effective modeling of dependencies between different types of variables improves the sample efficiency of BO.
We believe the generality of the approach can have a wider impact on modeling dependencies between discrete variables and variables of arbitrary other types, including continuous variables.


\bibliographystyle{plainnat}
\bibliography{oh_359}

\onecolumn

\title{Mixed Variable Bayesian Optimization with Frequency Modulated Kernels\\Supplementary Material}

\maketitle

\setcounter{section}{0}

\section{Positive definite FM kernels}
\label{supp:sec:thm_definite}
For a weighted undirected graph $\calG=(\calV,\calE)$ with graph Laplacian $L(\calG) = U \Lambda U^T$.
Frequency modulating kernels are defined as
\begin{equation}\label{supp:eq:frequency_modulating}
    k((\bfc, v),(\bfc', v') \Vert \bftheta, \beta) 
    = [\sum_{i=1}^{\Vert \calV \Vert} 
    [U]_{:,i} 
    f(\lambda_i, \Vert \bfc - \bfc' \Vert_{\bftheta} \Vert \beta) [U]_{:,i}]_{v, v'} 
\end{equation}
where $[U]_{:,i}$ are eigenvectors of $L(\calG)$ which are columns of $U$ and $\lambda_i = [\Lambda]_{ii}$ are corresponding eigenvalues.
$\bfc$ and $\bfc'$ are continuous variables in $\bbR^{D_{\calC}}$, $\bftheta \in \bbR^{D_{\calC}}$ is a kernel parameter similar to the lengthscales in the RBF kernel.
$\beta \in \bbR$ is a kernel parameter from kernels derived from the graph Laplacian.

\begin{theorem} \label{supp:thm:positive_definite}
    If $f(\lambda, \Vert \bfc - \bfc' \Vert_{\bftheta} \Vert \beta)$ defines a positive definite kernel on $(\bfc, \bfc') \in \bbR^{D_{\calC}} \times \bbR^{D_{\calC}}$, then a FreMod kernel defined with such $f$ is positive definite jointly on $(\bfc, v)$.
\end{theorem}

\begin{proof}
    \begin{equation}
        \small
        k((\bfc, v),(\bfc', v') \Vert \bftheta, \beta) 
        = \Big[\sum_{i=1}^{\Vert \calV \Vert} [U]_{:,i} f(\lambda_i, \Vert \bfc - \bfc' \Vert_{\bftheta} \Vert \beta) [U]_{:,i}\Big]_{v, v'} 
        = \sum_{i=1}^{\Vert \calV \Vert} [U]_{v,i} f(\lambda_i, \Vert \bfc - \bfc' \Vert_{\bftheta} \Vert \beta) [U]_{v',i}
    \end{equation}
    Since a sum of positive definite(PD) kernels is PD, we prove PD of frequency modulating kernels by showing that $k_i((\bfc, v),(\bfc', v') \Vert \bftheta, \beta) = [U]_{v,i} f(\lambda_i, \Vert \bfc - \bfc' \Vert_{\bftheta} \Vert \beta) [U]_{v',i}$ is PD.
    
    Let us consider $\bfa \in \bbR^S$, $\mathcal{D} = \{(\bfc_1,v_1), \cdots, (\bfc_S, v_S)\}$, then
    \begin{align}
        \label{supp:eq:psd_coard}
        &\bfa^T
        \begin{bmatrix}
        [U]_{v_1,i} f(\lambda_i, \Vert \bfc_1 - \bfc_1 \Vert_{\bftheta} \vert \beta) [U]_{v_1,i} & \cdots &
        [U]_{v_1,i} f(\lambda_i, \Vert \bfc_1 - \bfc_S \Vert_{\bftheta} \vert \beta) [U]_{v_S,i} \\
        [U]_{v_2,i} f(\lambda_i, \Vert \bfc_2 - \bfc_1 \Vert_{\bftheta} \vert \beta) [U]_{v_1,i} & \cdots &
        [U]_{v_2,i} f(\lambda_i, \Vert \bfc_2 - \bfc_S \Vert_{\bftheta} \vert \beta) [U]_{v_S,i} \\
        \vdots & \cdots & \vdots \\
        [U]_{v_S,i} f(\lambda_i, \Vert \bfc_S - \bfc_1 \Vert_{\bftheta} \vert \beta) [U]_{v_1,i} & \cdots &
        [U]_{v_S,i} f(\lambda_i, \Vert \bfc_S - \bfc_S \Vert_{\bftheta} \vert \beta) [U]_{v_S,i}
        \end{bmatrix}
        \bfa \nonumber \\
        & = (\bfa \circ [U]_{:,i})^T
        \begin{bmatrix}
        f(\lambda_i, \Vert \bfc_1 - \bfc_1 \Vert_{\bftheta} \Vert \beta) & \cdots & f(\beta \lambda_i, \Vert \bfc_1 - \bfc_S \Vert_{\bftheta} \vert \beta) \\
        f(\lambda_i, \Vert \bfc_2 - \bfc_1 \Vert_{\bftheta} \Vert \beta) & \cdots & f(\beta \lambda_i, \Vert \bfc_2 - \bfc_S \Vert_{\bftheta} \vert \beta) \\
        \vdots & \cdots & \vdots \\
        f(\lambda_i, \Vert \bfc_S - \bfc_1 \Vert_{\bftheta} \Vert \beta) & \cdots & f(\beta \lambda_i, \Vert \bfc_S - \bfc_S \Vert_{\bftheta} \vert \beta)
        \end{bmatrix}
        (\bfa \circ [U]_{:,i})
    \end{align}
    where $\circ$ is Hadamard(elementwise) product and $[U]_{:,i} = [[U]_{v_1,i}, \cdots, [U]_{v_S,i}]^T$.
    
    By letting $\bfa' = \bfa \circ [U_i]_{\pi_i(v_:),n}$, since $f(\lambda_i, \Vert \bfc - \bfc' \Vert_{\bftheta} \Vert \beta)$ is PD, we show that $k_i((\bfc, v),(\bfc', v') \Vert \bftheta, \beta) = u_{i,v} f(\lambda_i, \Vert \bfc - \bfc' \Vert_{\bftheta} \Vert \beta) u_{i,v'}$ is PD.
\end{proof}

\section{Nonnegative valued FM kernels}
\label{supp:sec:thm_positivity}
\begin{theorem} \label{supp:thm:nonnegative_valued}
    For a connected and undirected graph $\calG=(\calV,\calE)$ with non-negative weights on edges, define a kernel $k(v,v') = [Uf(\Lambda)U^T]_{v,v'}$ where $U$ and $\Lambda$ are eigenvectors and eigenvalues of the graph Laplacian $L(\calG)=U \Lambda U^T$. 
    If $f$ is any non-negative and strictly decreasing convex function on $[0, \infty)$, then $K(v,v') \ge 0$ for all $v, v' \in \calV$.
\end{theorem}
\begin{proof}
    For a connected and weighted undirected graph $\calG = (\calV, \calE)$, the graph Laplacian $L(G)$ has exactly one $0$ eigenvalue and the corresponding eigenvector $1/\sqrt{D}[1, \cdots, 1]^T$ when $\vert \calV \vert = D$.
    
    We show that
    \begin{equation}
        \min_{v, v'} k_{\calG}(v,v') = \min_{p,q = 1,\cdots,D} [U f(\Lambda) U^T]_{p,q} \ge 0
    \end{equation}
    for an arbitrary connected and weighted undirected graph $\calG = (\calV, \calE)$ where $\vert \calV \vert = D$ and $L(\calG) = U \Lambda U^T$.
    
    For a connected graph, there is only one zero eigenvalue
    \begin{equation}
        0 = \lambda_1 < \lambda_2 \le \cdots \le \lambda_D \quad \text{where} \quad \lambda_i = [\Lambda]_{i,i}
    \end{equation}
    and the corresponding eigenvector is given as
    \begin{equation}
        U_{1,q} = \frac{1}{\sqrt{D}}(q=1,\cdots,D).
    \end{equation}
    
    From the definition of eigendecomposition, we have
    \begin{equation}
        U^T U = U U^T = I.
    \end{equation}
    
    Importantly, from the definition of the graph Laplacian
    \begin{equation}
        [U \Lambda U^T]_{p,q} \le 0 \quad \text{when} \quad p \neq q.
    \end{equation}
        
    For a given diagonal matrix $\Lambda$ such that $0 = \lambda_1 < \lambda_2 \le \cdots \le \lambda_D$ where $\lambda_i = [\Lambda]_{i,i}$, 
    we solve the following minimization problem 
    \begin{equation}\label{supp:eq:nonnegativity_objective}
        \min_{[U]_{p,i}, [U]_{q,i}} \frac{f(0)}{D} + \sum_{i=2}^D f(\lambda_i) [U]_{p,i} [U]_{q,i}
    \end{equation}
    with the constraints
    \begin{equation}
        \sum_{i=2}^D \lambda_i [U]_{p,i} [U]_{q,i} \le 0 (p \neq q), \quad
        \sum_{i=2}^D [U]_{p,i}^2 = \sum_{i=2}^D [U]_{q,i}^2 = 1 - \frac{1}{D}, \quad \sum_{i=2}^D [U]_{p,i} [U]_{q,i} = - \frac{1}{D} (p \neq q)
    \end{equation}
    When $p = q$, eq.\ref{supp:eq:nonnegativity_objective} is nonnegative because $f$ is nonnegative valued.
    From now on, we consider the case $p \neq q$.
    
    Lagrange multiplier is given as
    \begin{align}
        L_{KKT}([U]_{p,i}, [U]_{q,i}, \eta, a, b, c) & = \frac{f(0)}{D} + \sum_{i=2}^D f(\lambda_i) [U]_{p,i} [U]_{q,i} + \eta \Big( \sum_{i=2}^D \lambda_i [U]_{p,i} [U]_{q,i} \Big) \nonumber \\
        + a \Big( \sum_{i=2}^D [U]_{p,i}^2 - (1 - \frac{1}{D}) \Big) & + b \Big( \sum_{i=2}^D [U]_{q,i}^2 - (1 - \frac{1}{D}) \Big) + c \Big( \sum_{i=2}^D [U]_{p,i} [U]_{q,i} + \frac{1}{D} \Big)
    \end{align}
    with $\eta \ge 0$.
    
    The stationary conditions are given as
    \begin{gather}
        \frac{\partial L_{KKT}}{\partial [U]_{p,i}} = f(\lambda_i) [U]_{q,i} + \eta \lambda_i [U]_{q,i} + c [U]_{q,i} + 2 a [U]_{p,i} = 0 \\
        \frac{\partial L_{KKT}}{\partial [U]_{q,i}} = f(\lambda_i) [U]_{p,i} + \eta \lambda_i [U]_{p,i} + c [U]_{p,i} + 2 b [U]_{q,i} = 0
    \end{gather}
    from which, we have
    \begin{gather}
        (f(\lambda_i) + \eta \lambda_i + c) [U]_{q,i} = - 2 a [U]_{p,i} \label{supp:eq:stationary_p} \\
        (f(\lambda_i) + \eta \lambda_i + c) [U]_{p,i} = - 2 b [U]_{q,i} \label{supp:eq:stationary_q}
    \end{gather}
    
    By using 
    \begin{equation}
        \sum_{i=2}^D \frac{\partial L_{KKT}}{\partial [U]_{p,i}} [U]_{p,i} = \sum_{i=2}^D \frac{\partial L_{KKT}}{\partial [U]_{q,i}} [U]_{q,i} = 0
    \end{equation}
    we have $a = b$.
    
    From eq.\eqref{supp:eq:stationary_p} and eq.\eqref{supp:eq:stationary_q},
    we get
    \begin{gather}
        ((f(\lambda_i) + \eta \lambda_i + c)^2 - 4ab) [U]_{q,i} = 0 \\
        ((f(\lambda_i) + \eta \lambda_i + c)^2 - 4ab) [U]_{p,i} = 0
    \end{gather}
    
    If $i \in \{i \vert (f(\lambda_i) + \eta \lambda_i + c)^2 - 4ab \neq 0 \}$, we have $[U]_{p,i} = [U]_{q,i} = 0$.
    On the other hand, if $(f(\lambda_i) + \eta \lambda_i + c)^2 - 4ab = 0$, then we have $f(\lambda_i) + \eta \lambda_i + c = -2a$ or $f(\lambda_i) + \eta \lambda_i + c = 2a$ because $a = b$.
    
    We define three index sets
    \begin{align}
        I_0 & = \{i \vert (f(\lambda_i) + \eta \lambda_i + c)^2 - 4a^2 \neq 0\} \\
        I_+ & = \{i \vert f(\lambda_i) + \eta \lambda_i + c + 2 a = 0\} - \{1\} \\
        I_- & = \{i \vert f(\lambda_i) + \eta \lambda_i + c - 2 a = 0\}
    \end{align}
    from eq.\eqref{supp:eq:stationary_p} and eq.\eqref{supp:eq:stationary_q}, we have
    \begin{align}
        i \in I_0 & \Rightarrow [U]_{p,i} = [U]_{q,i} = 0 \\
        i \in I_+ & \Rightarrow [U]_{p,i} = [U]_{q,i} \\
        i \in I_- & \Rightarrow [U]_{p,i} = -[U]_{q,i}
    \end{align}
    
    With these conditions, the constraints can be expressed as
    \begin{equation}
        \sum_{i_+ \in I_+} \lambda_{i_+} [U]_{p,i}^2 - \sum_{i_- \in I_-} \lambda_{i_-} [U]_{p,i}^2 \le 0, \quad \sum_{i_+ \in I_+} [U]_{p,{i_+}}^2 = \frac{1}{2} - \frac{1}{D}, \quad \sum_{i_- \in I_-} [U]_{p,{i_-}}^2 = \frac{1}{2}
    \end{equation}
    
    We divide cases according to the number of solutions $g(\lambda) = f(\lambda) + \eta \lambda$ can have. \textit{i)} $f(\lambda) + \eta \lambda$ can have at most one solution, \textit{ii)} $f(\lambda) + \eta \lambda$ may have two solutions. 
    Note that $g(\lambda)$ is convex as sum of two convex functions.
    Since a convex function can have at most two zeros unless it is constantly zero, these two cases are exhaustive.
    When $\eta = 0$, $f(\lambda)$ is strictly decreasing function and, thus $g(\lambda)$ has at most one solution. 
    Also, when $\eta \ge -f'(0) = \max_{\lambda} -f'(\lambda)$, $f'(\lambda) + \eta$ is positive except for $\lambda = 0$ and $g(\lambda)$ has at most one solution.
    
    \paragraph{Case \textit{i)}} $f(\lambda) + \eta \lambda$ can have at most one solution. ($\eta = 0$ or $\eta \ge -f'(0) = \max_{\lambda} -f'(\lambda)$)
    
    Let us denote $\lambda^E$ the unique solution of $f(\lambda_i) + \eta \lambda_i + c + 2 a = 0$ and $\lambda^N$ the unique of $f(\lambda_i) + \eta \lambda_i + c - 2 a = 0$.
    
    Therefore $\lambda_{i_+} = \lambda^E$, $\forall i_+ \in I_+$ and $\lambda_{i_-} = \lambda^N$, $\forall i_- \in I_-$.
    The minimization objective becomes
    \begin{align}
        \frac{f(0)}{D} + \sum_{i=2}^D f(\lambda_i) [U]_{p,i} [U]_{q,i} &= \frac{f(0)}{D} + f(\lambda^E) \sum_{i_+ \in I_+} [U]_{p,i}^2 - f(\lambda_N) \sum_{i_- \in I_-} [U]_{p,i}^2 \nonumber \\
        &= \frac{f(0)}{D} + \Big( \frac{1}{2} - \frac{1}{D} \Big) f(\lambda^E) - \frac{1}{2} f(\lambda^N)
    \end{align}
    The inequality constraint becomes
    \begin{align}
        \sum_{i=2}^D \lambda_i [U]_{p,i} [U]_{q,i} = \frac{f(0)}{D} + \lambda^E \sum_{i_+ \in I_+} [U]_{p,i_+}^2 - \lambda_N \sum_{i_- \in I_-} [U]_{p,i_-}^2 = \Big( \frac{1}{2} - \frac{1}{D} \Big) \lambda^E - \frac{1}{2} \lambda^N \le 0
    \end{align}
    
    Since $\lambda^E, \lambda^N \in \{\lambda_2, \cdots, \lambda_D\}$, there is maximum value with respect to the choice of $\lambda^E, \lambda^N$.
    We consider continuous relaxation of the minimization problem with respect to $\lambda^E, \lambda^N $.
    By showing that the objective is nonnegative when $\lambda^E \ge 0, \lambda^N \ge 0$, we prove our claim.
    When we consider continuous optimization problem over $\lambda^E, \lambda^N$, the minimum is obtained when the inequality constraint becomes equality constraints.
    If $\Big( \frac{1}{2} - \frac{1}{D} \Big) \lambda^E - \frac{1}{2} \lambda^N < 0$ by increasing $\lambda^E$ by $\delta > 0$ so that $\Big( \frac{1}{2} - \frac{1}{D} \Big) (\lambda^E + \delta) - \frac{1}{2} \lambda^N = 0$, $f(\lambda^E)$ is decreased to $f(\lambda^E + \delta)$, thus the minimum is obtained when the inequality constraint is equality.
    When $\eta > 0$, the inequality constraint automatically becomes an equality constraint by the slackness condition of the Karush-Kuhn-Tucker conditions.
    
    With the inequality condition the objective becomes
    \begin{equation}
        \frac{f(0)}{D} + \Big( \frac{1}{2} - \frac{1}{D} \Big) f(\lambda^E) - \frac{1}{2} f\Big((1 - \frac{2}{D}) \lambda^E\Big)
    \end{equation}
    taking derivative with respect to $\lambda_E$, we have
    \begin{equation}
        \Big( \frac{1}{2} - \frac{1}{D} \Big) \Big(f'(\lambda^E) - f'((1 - \frac{2}{D}) \lambda^E)\Big)
    \end{equation}
    
    By the convexity of $f$, the derivative is always nonnegative with respect to $\lambda^E \ge 0$.
    
    Since
    \begin{equation}
        \lim_{\lambda_E \rightarrow 0} \frac{f(0)}{D} + \Big( \frac{1}{2} - \frac{1}{D} \Big) f(\lambda^E) - \frac{1}{2} f((1 - \frac{2}{D}) \lambda^E) = 0
    \end{equation}
    
    The minimum is nonnegative.
    
    \paragraph{Case \textit{ii)}} $f(\lambda) + \eta \lambda$ may have two solutions. ($0 < \eta < -f'(0) = \max_{\lambda} -f'(\lambda)$)
    
    By the slackness condition, the inequality constraint becomes an equality constraint. 
    Since $f(\lambda) + \eta \lambda$ is convex, it has at most two solutions.
    Let us denote $\lambda_1^E < \lambda_2^E$ two solutions of $f(\lambda) + \eta \lambda + c + 2 a = 0$ and $\lambda_1^N < \lambda_2^N$ two solutions of $f(\lambda) + \eta \lambda + c - 2 a = 0$
    Then 
    \begin{align}
        f(\lambda_1^E) + \eta \lambda_1^E + c + 2 a = 0 \\
        f(\lambda_2^E) + \eta \lambda_2^E + c + 2 a = 0 \\
        f(\lambda_1^N) + \eta \lambda_1^N + c - 2 a = 0 \\
        f(\lambda_2^N) + \eta \lambda_2^N + c - 2 a = 0
    \end{align}
    The objective becomes
    \begin{align}
        \frac{f(0)}{D} &
        + f(\lambda_1^E) \sum_{i_+ \in I_+:\lambda_{i_+} = \lambda_1^E} [U]_{p,i_+}^2 + f(\lambda_2^E) \sum_{i_+ \in I_+:\lambda_{i_+} = \lambda_2^E} [U]_{p,i_+}^2
        \nonumber \\
        &
        - f(\lambda_1^N) \sum_{i_- \in I_-:\lambda_{i_-} = \lambda_1^N} [U]_{p,i_-}^2 - f(\lambda_2^N) \sum_{i_- \in I_-:\lambda_{i_-} = \lambda_2^N} [U]_{p,i_-}^2
    \end{align}
    with the constraints
    \begin{gather}
        \sum_{i_+ \in I_+:\lambda_{i_+} = \lambda_1^E} [U]_{p,i_+}^2 
        + \sum_{i_+ \in I_+:\lambda_{i_+} = \lambda_2^E} [U]_{p,i_+}^2 
        = \frac{1}{2} - \frac{1}{D} \\
        \sum_{i_- \in I_-:\lambda_i = \lambda_1^N} [U]_{p,i_-}^2
        + \sum_{i_- \in I_-:\lambda_i = \lambda_2^N} [U]_{p,i_-}^2 = \frac{1}{2} \\
        \lambda_1^E \sum_{i_+ \in I_+:\lambda_{i_+} = \lambda_1^E} [U]_{p,i_+}^2 
        + \lambda_2^E \sum_{i_+ \in I_+:\lambda_{i_+} = \lambda_2^E} [U]_{p,i_+}^2 
        - \lambda_1^N \sum_{i_- \in I_-:\lambda_{i_-} = \lambda_1^N} [U]_{p,i_-}^2 
        - \lambda_2^N \sum_{i_- \in I_-:\lambda_{i_-} = \lambda_2^N} [U]_{p,i_-}^2 = 0
    \end{gather}
    Let
    \begin{equation}
        A^E = \sum_{i_+ \in I_+:\lambda_{i_+} = \lambda_1^E} [U]_{p,i_+}^2 \in [0, \frac{1}{2}-\frac{1}{D}], \quad 
        A^N = \sum_{i_- \in I_-:\lambda_{i_-} = \lambda_1^N} [U]_{p,i_-}^2 \in [0, \frac{1}{2}]
    \end{equation}
    Then the objective becomes
    \begin{equation}
        \frac{f(0)}{D} + f(\lambda_1^E) A^E + f(\lambda_2^E) (\frac{1}{2}-\frac{1}{D} - A^E) - f(\lambda_1^N) A^N - f(\lambda_2^N) (\frac{1}{2} - A^N)
    \end{equation}
    Taking derivatives
    \begin{gather}
        \frac{\partial}{\partial A^E} \Rightarrow f(\lambda_1^E) - f(\lambda_2^E) > 0 \\
        \frac{\partial}{\partial A^N} \Rightarrow -f(\lambda_1^N) + f(\lambda_2^N) < 0 \\
    \end{gather}
    Thus the minimum is obtained at the boundary point where $A^E = 0$ and $A^N = \frac{1}{2}$ which falls back to Case \textit{i)} whose minimum is bounded below by zero.

\end{proof}

\begin{remark}
    Theorem~\ref{thm:nonnegative_valued} holds for weighted undirected graphs, that is, for any arbitrary graph with arbitrary symmetric nonnegative edge weights.
\end{remark}
\begin{remark}
    Note that in numerical simulations, you may observe small negative values ($\approx 10^{-7}$) due to numerical instability.
\end{remark}
\begin{remark}
    In numerical simulations, the convexity condition does not appear to be necessary for complete graphs where $\max_{p \neq q}[L(\calG)]_{p,q} < -\epsilon$ for some $\epsilon > 0$.
    For complete graphs, the convexity condition may be relaxed, at least, in a stochastic sense.
\end{remark}

\begin{corollary} \label{supp:coro:nonnegative_valued}
    The random walk kernel derived from normalized Laplacian~\cite{smola2003kernels} and the diffusion kernels~\cite{kondor2002diffusion}, the ARD diffusion kernel~\cite{oh2019combinatorial} and the regularized Laplacian kernel~\cite{smola2003kernels} derived from normalized and unnormalized Laplacian are all positive valued kernels.
\end{corollary}

\begin{proof}
    The condition that off-diagonal entries are nonpositive holds for both normalized and unnormalized graph Laplacian.
    Therefore for normalized graph Laplacian, the proof in the above theorem can be applied without modification.
    The positivity of kernel value also holds for kernels derived from normalized Laplacian as long as it satisfies the conditions in Thm.\ref{thm:nonnegative_valued}.
\end{proof}

\begin{remark}
    In numerical simulations with nonconvex functions and arbitrary connected and  weighted undirected graphs, negative values easily occur.
    For example, the inverse cosine kernel~\cite{smola2003kernels} does not satisfies the convexity condition and has negative values.
\end{remark}

\section{Examples of FM kernels}
\label{supp:sec:FM_examples}
In this section, we first review the definition of conditionally negative definite(CND) and relations between positive definite(PD).
Utilizing relations between PD and CND and properties of PD and CND, we provide an example of a flexible family of frequency modulating functions.

\begin{definition}[3.1.1~\citep{berg1984harmonic}]\label{supp:def:cnd}
    A symmetric function $k : \calX \times \calX \rightarrow \bbR$ is called a conditionally negative definite(CND) kernel if ~$\forall n \in \bbN$, $x_1, \cdots, x_n \in \calX$ $a_1, \cdots, a_n \in \bbR$ such that $\sum_{i=1}^n a_i = 0$
    \begin{equation}
        \sum_{i,j=1}^n a_i k(x_i, x_j) a_j \le 0
    \end{equation}
\end{definition}
Please note that CND requires the condition $\sum_{i=1}^n a_i = 0$.

\begin{theorem}[3.2.2~\citep{berg1984harmonic}]\label{supp:thm:pd_cnp_exponential}
    $K(x,x')$ is conditionally negative definite if and only if $e^{-tK(x,x')}$ is positive definite for all $t > 0$.
\end{theorem}

As mentioned in p.75~\citep{berg1984harmonic}, from Thm.~\ref{supp:thm:pd_cnp_exponential}, we have

\begin{theorem} \label{supp:thm:pd_cnp_reciprocal}
    $K(x,x')$ is conditionally negative definite and $K(x, x')) \ge 0$ if and only if $(t + K(x, x'))^{-1}$ is positive definite for all $t > 0$.
\end{theorem}

\begin{theorem}[3.2.10~\citep{berg1984harmonic}]
    If $K(x,x')$ is conditionally negative definite and $K(x, x) \ge 0$, then $(K(x,x'))^a$ for $0 < a < 1$ and $\log{K(x,x')}$ are conditionally negative definite.
\end{theorem}

\begin{theorem}[3.2.13~\citep{berg1984harmonic}]\label{supp:thm:cnp_inner_prod}
    $K(x,x') = \Vert x - x' \Vert^p$ is conditionally negative definite for all $0 < p \le 2$.
\end{theorem}

Using above theorems, we provide a quite flexible family of frequency modulating functions

\begin{proposition}\label{supp:prop:FM_family}
    For $S \in (0, \infty)$, a finite measure $\mu$ on $[0, S]$ and $\mu$-measurable $\tau : [0, S] \rightarrow [0, 2]$ and $\rho : [0, S] \rightarrow \bbN$,
    \begin{equation}
        f(\lambda, \Vert \bfc - \bfc' \Vert_{\bftheta} \vert \alpha, \beta) 
        = \int_0^S \frac{1}{(1 + \beta \lambda + \alpha \Vert \bfc - \bfc' \Vert_{\bftheta}^{\tau(s)})^{\rho(s)}} \mu(ds)
    \end{equation}
    is a frequency modulating function.
\end{proposition}

\begin{proof}
    First we show that
    \begin{equation}
        f^{p,t}(\lambda, \Vert \bfc - \bfc' \Vert_{\bftheta} \vert \alpha, \beta) = \frac{1}{(1 + \beta \lambda + \alpha \Vert \bfc - \bfc' \Vert_{\bftheta}^t)^p}
    \end{equation}
    is a frequency modulating function for $t \in (0, 2]$ and $p \in \bbN$.
    
    Property \textbf{FM-P1} on $f^{p,t}$) $f^{p,t}(\lambda, \Vert \bfc - \bfc' \Vert_{\bftheta} \vert \alpha, \beta)$ is positive valued and decreasing with respect to $\lambda$.
    
    Property \textbf{FM-P2} on $f^{p,t}$) $\Vert \bfc - \bfc' \Vert_{\bftheta}$ is conditionally negative definite by Thm.\ref{supp:thm:cnp_inner_prod}
    Then by Thm.\ref{supp:thm:pd_cnp_reciprocal}, $\frac{1}{(1 + \beta \lambda + \alpha \Vert \bfc - \bfc' \Vert_{\bftheta}^t)}$ is positive definite with respect to $\bfc$ and $\bfc'$.
    Since the product of positive definite kernels is positive definite, $f^{p,t}(\lambda, \Vert \bfc - \bfc' \Vert_{\bftheta} \vert \alpha, \beta)$ is positive definite.
    
    Property \textbf{FM-P3} on $f^{p,t}$) Let $h^{p,t} = f^{p,t}(\lambda, \Vert \bfc - \bfc' \Vert_{\bftheta} \vert \alpha, \beta) - f^{p,t}(\lambda, \Vert \Tilde{\bfc} - \Tilde{\bfc}' \Vert_{\bftheta} \vert \alpha, \beta)$, then
    \begin{align}
        h_{\lambda}^{p,t} = \frac{\partial h^{p,t}}{\partial \lambda} &= -p \beta \Big(
        \frac{1}{(1 + \beta \lambda + \alpha \Vert \bfc - \bfc' \Vert_{\bftheta}^t)^{p+1}} - 
        \frac{1}{(1 + \beta \lambda + \alpha \Vert \Tilde{\bfc} - \Tilde{\bfc}' \Vert_{\bftheta}^t)^{p+1}} \Big) \nonumber \\
        h_{\lambda\lambda}^{p,t} = \frac{\partial^2 h^{p,t}}{\partial \lambda^2} &= p(p+1) \beta^2 \Big(
        \frac{1}{(1 + \beta \lambda + \alpha \Vert \bfc - \bfc' \Vert_{\bftheta}^t)^{p+2}} - 
        \frac{1}{(1 + \beta \lambda + \alpha \Vert \Tilde{\bfc} - \Tilde{\bfc}' \Vert_{\bftheta}^t)^{p+2}} \Big)
    \end{align}
    For $\Vert \bfc - \bfc' \Vert_{\bftheta} < \Vert \Tilde{\bfc} - \Tilde{\bfc}' \Vert_{\bftheta}$, $h > 0$, $h_{\lambda} < 0$ and $h_{\lambda\lambda} > 0$, therefore this satisfies the frequency modulation principle.
    
    Now we show that 
    \begin{equation}
        f(\lambda, \Vert \bfc - \bfc' \Vert_{\bftheta} \vert \alpha, \beta) 
        = \int_0^S \frac{1}{(1 + \beta \lambda + \alpha \Vert \bfc - \bfc' \Vert_{\bftheta}^{\tau(s)})^{\rho(s)}} \mu(ds)
    \end{equation}
    satisfies all 3 conditions.
    
    Property \textbf{FM-P1}) Trivial from the definition.
    
    Property \textbf{FM-P2}) Since a measurable function can be approximated by simple functions~\citep{folland1999real}, we approximate $f(\lambda, \Vert \bfc - \bfc' \Vert_{\bftheta} \vert \alpha, \beta)$ with following increasing sequence
    \begin{align}
        f_n(\lambda, \Vert \bfc - \bfc' \Vert_{\bftheta} \vert \alpha, \beta) =
        &\sum_{i=1}^{2^n} \sum_{j=1}^{n}  \frac{\mu(A_{i,j})}{(1 + \beta \lambda + \alpha \Vert \bfc - \bfc' \Vert_{\bftheta}^{\frac{i-1}{2^n}2})^j} \\
        &\text{where} \quad A_{i,j} = \{s \vert \frac{i-1}{2^n}2 < \rho(s) \le \frac{i}{2^n}2, \tau(s) = j \} \nonumber 
    \end{align}

    Each summand $\mu(A_{i,j}) / (1 + \beta \lambda + \alpha \Vert \bfc - \bfc' \Vert_{\bftheta}^{\frac{i-1}{2^n}2})^j$ is positive definite as shown above and sum of positive definite kernels is positive definite. 
    Therefore, $f_n(\lambda, \Vert \bfc - \bfc' \Vert_{\bftheta} \vert \alpha, \beta)$ is positive definite. 
    Since the pointwise limit of positive definite kernels is a kernel~\citep{fukumizu2010kernel}, we show that $f(\lambda, \Vert \bfc - \bfc' \Vert_{\bftheta} \vert \alpha, \beta)$ is positive definite.
    
    Property \textbf{FM-P3}) If we show that $\frac{\partial}{\partial \lambda}$ and $\int \mu(ds)$ are interchangeable, from the Condition \#3 on $f_{p,t}$, we show that $f(\lambda, \Vert \bfc - \bfc' \Vert_{\bftheta} \vert \alpha, \beta)$ satisfies the frequency modulating principle.
    
    Let $h = f(\lambda, \Vert \bfc - \bfc' \Vert_{\bftheta} \vert \alpha, \beta) - f(\lambda, \Vert \Tilde{\bfc} - \Tilde{\bfc}' \Vert_{\bftheta} \vert \alpha, \beta)$.
    There is a constant $A > 0$ such that 
    \begin{equation}
        \Big\vert \frac{h^{\tau(s),\rho(s)}(\lambda + \delta) - h^{\tau(s),\rho(s)}(\lambda)}{\delta} \Big\vert
        < \Big\vert \frac{\partial h^{\tau(s),\rho(s)}}{\partial \lambda} \Big\vert + A
        < \Big\vert \frac{\partial h^{0,1}}{\partial \lambda} \Big\vert + A
    \end{equation}
    For a finite measure,
    $\Big\vert \frac{\partial h^{0,1}}{\partial \lambda} \Big\vert + A$ is integrable.
    Therefore, $\frac{\partial}{\partial \lambda}$ and $\int \mu(ds)$ are interchangeable by dominated convergence theorem~\citep{folland1999real}.
    With the same argument, $\frac{\partial^2}{\partial \lambda^2}$ and $\int \mu(ds)$ are interchangeable.
    
    Now, we have
    \begin{align}
        h_{\lambda} = \frac{\partial h}{\partial \lambda} &= \int_0^S \frac{\partial h^{\tau(s),\rho(s)}}{\partial \lambda} \mu(ds) \nonumber \\
        h_{\lambda\lambda} = \frac{\partial^2 h}{\partial \lambda^2} &= \int_0^S \frac{\partial^2 h^{\tau(s),\rho(s)}}{\partial \lambda^2} \mu(ds) \nonumber
    \end{align}
    From the Condition \#3 on $f^{p,t}$, $h_{\lambda} < 0$ and $h_{\lambda\lambda} > 0$ follow and thus we show that $f(\lambda, \Vert \bfc - \bfc' \Vert_{\bftheta} \vert \alpha, \beta)$ satisfies the frequency modulating principle.
    
    $f(\lambda, \Vert \bfc - \bfc' \Vert_{\bftheta} \vert \alpha, \beta)$ is a frequency modulating function.
\end{proof}

\begin{proposition}\label{supp:prop:fm_kernel_extension_PD}
    If $k_{\calH} : \calH \times \calH \rightarrow \bbR$ on a RKHS $\calH$ is bounded above by $u > 0$, then for any $\delta > 0$
    \begin{equation}
        f(\lambda, k_{\calH}(h,h') \vert \alpha, \beta) = \frac{1}{\delta + u + \beta \lambda - k_{\calH}(h,h')}
    \end{equation}
    is positive definite on $(h,h') \in \calH \times \calH$.
\end{proposition}
\begin{proof}
    The negation of a positive definite kernel is conditionally negative definite by Supp.~Def.~\ref{supp:def:cnd}.
    Also, by definition, a constant plus a conditionally negative definite kernel is conditionally negative definite.
    Therefore, $u - k_{\calH}(h,h')$ is conditionally negative definite.
    
    Using Supp.~Thm.\ref{supp:thm:pd_cnp_reciprocal}, we show that $1/(\delta + u + \beta \lambda - k_{\calH}(h,h'))$ is positive definite on $(h,h') \in \calH \times \calH$.
\end{proof}

\section{Experimental Details}
\label{supp:sec:exp_details}
In this section, we provide the details of each component of BO pipeline, the surrogate model and how it is fitted to evaluation data, the acquisition function and how it is optimized.
We also provide each experiment specific details including the search spaces, evaluation detail, run time analysis and etc.
The code used for the experiments will be released upon acceptance.


\subsection{Acquisition Function Optimization}\label{supp:subsec:acquisition}
We use Expected Improvement~(EI) acquisition function~\citep{donald1998efficient}.
Since, in mixed variable BO, acquisition function optimization is another mixed variable optimization task, we need a procedure to perform an optimization of acquisition functions on mixed variables.

\paragraph{Acquisition Function Optimization}\label{supp:subsec:par:acquisition_optimization}
Similar to~\citep{daxberger2019mixed}, we alternatively call continuous optimizer and discrete optimizer, which is similar to coordinate-wise ascent, and, in this case, it is so-called type-wise ascent.
For continuous variables, we use L-BFGS-B~\citep{zhu1997algorithm} and for discrete variables, we use hill climbing~\citep{skiena1998algorithm}.
Since the discrete part of the search space is represented by graphs, hill climbing is amount to greedy ascent in neighborhood.
We alternate one discrete update using hill climbing call and one continuous update by calling $scipy.optimize.minimize(method=\text{"L-BFGS-B"},maxiter=1)$.

\paragraph{Spray Points}\label{supp:subsec:par:spray_points}
Acquisition functions are highly multi-modal and thus initial points with which the optimization of acquisition functions starts have an impact on exploration-exploitation trade-off.
In order to encourage exploitation, spray points~\citep{snoek2012practical,garnett2010bayesian,oh2018bock}, which are points in the neighborhood of the current optimum~(e.g, optimum among the collected evaluations), has been widely used.

\paragraph{Initial points for acquisition function optimization}
On 50 spray points and 100000 randomly sampled points, acquisition values are computed, and the highest 40 are used as initial points to start acquisition function optimization.

\subsection{Joint optimization of neural architecture and SGD hyperparameter}\label{supp:subsec:joint_NAS}

\paragraph{Discrete Part of the Search Space}
The discrete part of the search space, $\mathcal{A}$, is modified from the NASNet search space~\citep{zoph2016neural}. 
Each block consists of 4 states $S_1, S_2, S_3, S_4$ and takes two inputs $S_{-1}, S_0$ from a previous block. 
For each state, two inputs are chosen from the previous states, 
Then two operations are chosen and the state finishes its process by summing up two results of the chosen operation
For example, if two inputs $S_{-1}$, $S_2$ and two operations $OP_3^{(1)}$, $OP_3^{(2)}$ are chosen for $S_3$, we have $(S_{-1},S_2) \xrightarrow{S_3} OP_3^{(1)}(S_{-1}) + OP_3^{(2)}(S_2)$.

Operations are chosen from 8 types below
\vspace{-4pt}
\begin{multicols}{4}
\begin{itemize}
    \item ID
    \item Conv$1\times1$
    \item Conv$3\times3$
    \item Conv$5\times5$
    \item Separable Conv$3\times3$
    \item Separable Conv$5\times5$
    \item Max Pooling$3\times3$
    \item Max Pooling$5\times5$
\end{itemize}
\end{multicols}
Two inputs for each state are chosen from states with smaller subscript(e.g $S_i$ is allowed to have $S_j$ as an input if $j < i$).
By choosing $S_4$ and one of $S_1, S_2, S_3$ as outputs of the block, the configuration of a block is completed.

In \textsc{ModLap}, it is required to specify graphs for discrete variables.
For graphs representing operation types, we use complete graphs.
For graphs representing inputs of each states, we use graphs which reflect the ordering structure.
In a graph representing inputs of each state, each vertex is represented by a tuple, for the graph representing inputs of $S_3$, it has a vertex set of $\{(-1, 0), (-1, 1), (-1, 2), (0, 1), (0, 2), (1, 2)\}$.
For example, choosing $(-1, 0)$ means $S_3$ takes $S_{-1}$(input 1 of the block) and $S_0$(input 2 of the block) as inputs of the cell and choosing $(0, 2)$ means $S_3$ takes $S_0$(input 2 of the block) and $S_2$(cell 2) as inputs.
There exists an edge between vertices as long as one input is shared and two distinct inputs differ by one.
For example, there is an edge between $(-1, 0)$ and $(-1, 1)$ because $-1$ is shared and $\vert 0 - 1 \vert = 1$ and there is no edge between $(-1, 0)$ and $(-1, 2)$ because $\vert 0 - 1 \vert \neq 1$ even though $-1$ is shared.
Note that in the graph representing inputs for $S_4$, we exclude the vertex $(-1, 0)$ to avoid the identity block.
For graphs representing outputs of the block, we use the path graph with 3 vertices since we restrict the output is one of $(1, 4), (2, 4), (3, 4)$.
By defining graphs corresponding variables in this way, a prior knowledge about the search space can be infused and be of help to Bayesian optimization.

\paragraph{Continuous Part of the Search Space}
The space of continuous hyperparameters $\mathcal{H}$ comprises 6 continuous hyperparameters of the SGD with a learning rate scheduler: learning rate, momentum, weight decay, learning rate reduction factor, 1st reduction point ratio1 and 2nd reduction point ratio. 
The ranges for each hyperparameter are given in Supp. Table~\ref{supp:tab:sgd_hp_range}.

\begin{table}[ht]
    \caption{SGD Hyperparameter Range}
    \label{supp:tab:sgd_hp_range}
    \vspace{-4pt}
    \centering
    \begin{tabular}{|c|c|c|c|c|}
        \toprule
        SGD hyperparameter & Transformation & Range \\
        \midrule
        Learning Rate & $\log$ & $[\log(0.001), \log(0.1)]$ \\
        Momentum & $\cdot$ & $[0.8, 1.0]$ \\
        Weight Decay & $\log$ & $[\log(10^{-6}), \log(10^{-2})]$ \\
        Learning Rate Reduction Factor & $\cdot$ & $[0.1, 0.9]$ \\
        1st Reduction Point Ratio & $\cdot$ & $[0, 1]$ \\
        2nd Reduction Point Ratio & $\cdot$ & $[0, 1]$ \\
        \bottomrule
    \end{tabular}
\end{table}

For a given learning rate $l$, learning rate reduction factor $\gamma$, 1st reduction point ratio $r_1$ and 2nd reduction point ratio $r_2$, then learning rate scheduling is given in Supp.~Table~\ref{supp:tab:lr_scheduling}.
\begin{table}[ht]
    \caption{Learning Rate Scheduling. In the experiment, the number of epochs $E$ is set to 25.}
    \label{supp:tab:lr_scheduling}
    \vspace{-4pt}
    \centering
    \begin{tabular}{|c|c|c|c|c|}
        \toprule
        Begin Epoch($<$) & ($\le$)End Epoch & Learning Rate \\
        \midrule
        0 & $E \times r_1$ & $l$ \\
        $E \times r_1$ & $E \times (r_1 + (1 - r_1) r_2)$ & $l \cdot \gamma$ \\
        $E \times (r_1 + (1 - r_1) r_2)$ & $E$ & $l \cdot \gamma^2$ \\
        \bottomrule
    \end{tabular}
    \vspace{-8pt}
\end{table}

\paragraph{Evaluation}
For a given block configuration $a \in \mathcal{A}$, the model is built by stacking 3 blocks with downsampling between blocks.
Note that there are two inputs and two outputs of the blocks.
Therefore, the downsampling is applied separately to each output.
The two outputs of the last block are concatenated after max pooling and then fed to the fully connected layer.

The model is trained with the hyperparameter $h \in \mathcal{H}$ on a half of FashionMNIST~\citep{xiao2017fashion} training data for 25 epochs and the validation error is computed on the rest half of training data.
To reduce the high noise in validation error, the validation error is averaged over 4 validation errors from models trained with different random initialization.
With the batch size of 32, each evaluation takes 12$\sim$21 minutes on a single GTX 1080 Ti depending on architectures

\paragraph{Regularized Evolution Hyperparameters}
RE has hyperparameters, the population size and the sample size.
We set to 50 and 15, respectively, to make those similar to the optimal choice in~\citep{real2019regularized,oh2019combinatorial}.
Accordingly, RE starts with a population with 50 random initial points.
In each run of 4 runs, the first 10 initial points of 50 random initial points are shared with 10 initial points used in GP-BO.

Another hyperparameter is the mutation rule.
In addition to the mutation of architectures used in~\citep{real2019regularized}, for continuous variables, a randomly chosen single continuous variable is mutated by Gaussian noise with small variance.
In each round, one continuous variable and one discrete variable are altered.

\paragraph{Wall-clock Run Time}
The total run time of \textsc{ModLap}(200), $61.44 \pm 4.09$ hours, is sum of $9.27 \pm 2.60$ hours for BO suggestions and $52.16 \pm 1.79$ hours for evaluations.
BO suggestions were run on Intel Xeon Processor E5-2630 v3 and evaluations were run on GTX 1080 Ti.

In the actual execution of RE, two different types of GPUs were used, GTX 1080 Ti(fast) and GTX 980(slow).
Therefore, the evaluation time for RE is estimated by assuming that RE were also run on GTX 1080 Ti(fast) only.
During the total run time of \textsc{ModLap}(200), $61.44 \pm 4.09$ hours, RE is estimated to collects 230 evaluations.
$230 \approx 61.44 / 52.16 \times (200 - 10) + 10$ where 10 is adjusted because the evaluation time for 10 random initial points was not measured.

Since in both RE and BOHB, we assume zero seconds to acquire new hyperparameters and only consider times spent for evaluations, the wall-clock runtime of BOHB is estimated to be equal to wall-clock runtime of RE.

\section{Experiment: Results}
\label{supp:sec:exp_results}

In this section, in addition to the results reported in Sec.~\ref{sec:experiment}, we provide additional results.

On 3 synthetic problems and 2 hyperparameter optimization problems, along with the frequency modulation, we also compare other kernel combinations such as the kernel addition and the kernel product as follows.
\begin{table}[h]
    \vspace{-6pt}
    \label{supp:tab:kernels_in_comparison}
    \footnotesize
    \begin{center}
        \begin{tabular}{|c|c|c|}
            \toprule
            \textsc{ProdLap} : $k_{RBF} \times k_{Lap}$ & 
            \textsc{AddLap} : $k_{RBF} + k_{Lap}$ & 
            \textsc{ModLap} : Eq.\ref{eq:scalable_FM} with $f = f_{Lap}$\\
            \midrule
            \textsc{ProdDif} : $k_{RBF} \times k_{Dif}$ & 
            \textsc{AddDif} : $k_{RBF} + k_{Dif}$ & 
            \textsc{ModDif} : Eq.\ref{eq:scalable_FM} with $f = f_{Dif}$\\
            \bottomrule
        \end{tabular}
    \end{center}
    where $k_{RBF}$ is the RBF kernel and
    \vspace{-10pt}
    \begin{center}
        \vspace{-10pt}
        \begin{equation}\label{supp:eq:discrete_kernel}
            k_{Lap}(\bfv, \bfv')
            = \prod_{p=1}^P \sum_{i=1}^{\vert \calV_p \vert} [U^p]_{v_p,i} \frac{1}{1 + \beta_p \lambda_i^p} [U^p]_{v_p',i}
            \quad
            k_{Dif}(\bfv, \bfv')
            = \prod_{p=1}^P \sum_{i=1}^{\vert \calV_p \vert} [U^p]_{v_p,i} \exp(-\beta_p \lambda_i^p) [U^p]_{v_p',i}
        \end{equation}
    \end{center}
    \vspace{-14pt}
\end{table}

We make following observations with this additional comparison.
Firstly, \textsc{ModDif} which does not respect the similarity measure behavior, sometimes severely degrades BO performance.
Secondly, the kernel product often performs better than the kernel addition.
Thirdly, \textsc{ModLap} shows the equally good final results as the kernel product and finds the better solution faster than the kernel product consistently. 
This can be clearly shown by comparing the area above the mean curve of BO runs using different kernels.
The area above the mean curve of BO using \textsc{ModLap} is larger than the are above the mean curve of BO using the kernel product.
Moreover, the gap between the area from \textsc{ModLap} and the area from kernel product increases in problems with larger search spaces. 
Even on the smallest search space, Func2C, \textsc{ModLap} lags behind the kernel product up to around 90th evaluation and outperforms after it.
The benefit of \textsc{ModLap} modeling complex dependency among mixed variables is more prominent in higher dimension problems.

On the joint optimization of SGD hyperparameters and architecture, we show the additional result where RE and BOHB are continued 600 evaluations.

\newpage

\subsection{Func2C}

\begin{table}[!h]
    \begin{minipage}{0.5\linewidth}
        \centering
        \includegraphics[width=0.9\columnwidth]{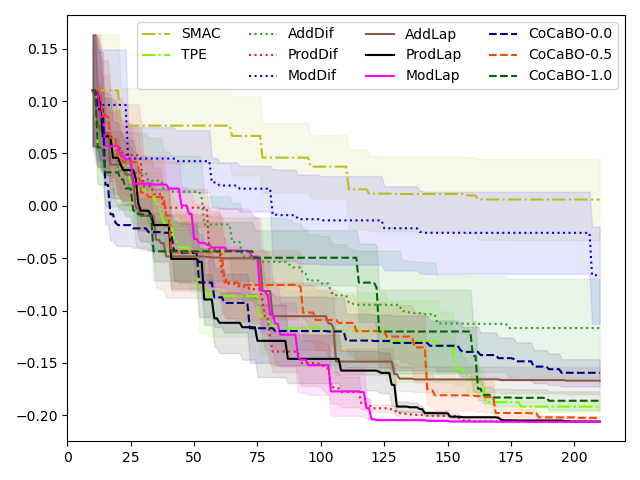}
    \end{minipage}
    \hfill
    \begin{minipage}{0.5\linewidth}
        \centering
        \begin{tabular}{lc}\hline
            Method             & Mean$\pm$Std.Err.     \\ \hline
            SMAC               & $  +0.0060\pm 0.0387$ \\
            TPE                & $  -0.1917\pm 0.0053$ \\ \hline
            AddDif             & $  -0.1167\pm 0.0472$ \\
            ProdDif            & $  -0.2060\pm 0.0002$ \\
            ModDif             & $  -0.0662\pm 0.0463$ \\ \hline
            AddLap             & $  -0.1669\pm 0.0127$ \\
            ProdLap            & $  -0.2060\pm 0.0001$ \\
            ModLap             & $  -0.2063\pm 0.0000$ \\ \hline
            CoCaBO-0.0         & $  -0.1594\pm 0.0130$ \\
            CoCaBO-0.5         & $  -0.2025\pm 0.0018$ \\
            CoCaBO-1.0         & $  -0.1861\pm 0.0090$ \\
            \hline
        \end{tabular}
    \end{minipage}
    \label{supp:fig:exp-func2c}
\end{table}

\subsection{Func3C}

\begin{table}[!h]
    \begin{minipage}{0.5\linewidth}
        \centering
        \includegraphics[width=0.9\columnwidth]{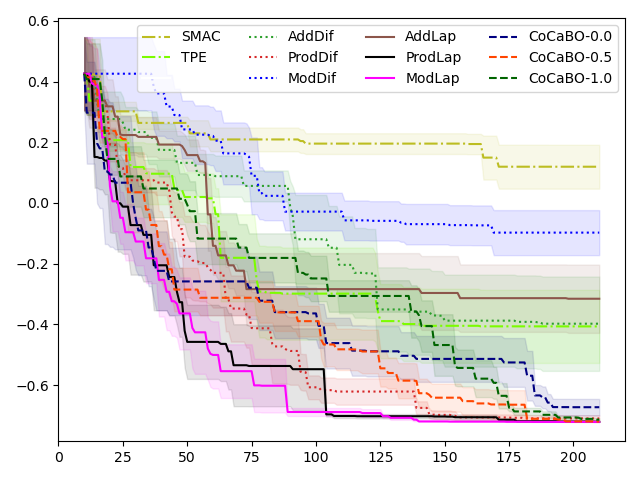}
    \end{minipage}
    \hfill
    \begin{minipage}{0.5\linewidth}
        \centering
        \begin{tabular}{lc}\hline
            Method             & Mean$\pm$Std.Err.     \\ \hline
            SMAC               & $  +0.1194\pm 0.0723$ \\
            TPE                & $  -0.4068\pm 0.1204$ \\ \hline
            AddDif             & $  -0.3979\pm 0.1555$ \\
            ProdDif            & $  -0.7100\pm 0.0106$ \\
            ModDif             & $  -0.0977\pm 0.0742$ \\ \hline
            AddLap             & $  -0.3156\pm 0.1125$ \\
            ProdLap            & $  -0.7213\pm 0.0005$ \\
            ModLap             & $  -0.7215\pm 0.0004$ \\ \hline
            CoCaBO-0.0         & $  -0.6730\pm 0.0274$ \\
            CoCaBO-0.5         & $  -0.7202\pm 0.0016$ \\
            CoCaBO-1.0         & $  -0.7139\pm 0.0051$ \\
            \hline
        \end{tabular}
    \end{minipage}
    \label{supp:fig:exp-func3c}
\end{table}

\subsection{Ackley5C}

\begin{table}[!h]
    \begin{minipage}{0.5\linewidth}
        \centering
        \includegraphics[width=0.9\columnwidth]{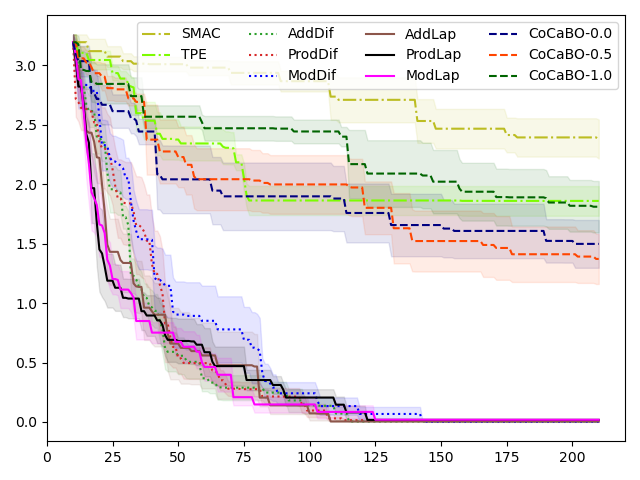}
    \end{minipage}
    \hfill
    \begin{minipage}{0.5\linewidth}
        \centering
        \begin{tabular}{lc}\hline
            Method             & Mean$\pm$Std.Err.     \\ \hline
            SMAC               & $  +2.3809\pm 0.1648$ \\
            TPE                & $  +1.8601\pm 0.1248$ \\ \hline
            AddDif             & $  +0.0040\pm 0.0015$ \\
            ProdDif            & $  +0.0152\pm 0.0044$ \\
            ModDif             & $  +0.0008\pm 0.0003$ \\ \hline
            AddLap             & $  +0.0042\pm 0.0018$ \\
            ProdLap            & $  +0.0177\pm 0.0038$ \\
            ModLap             & $  +0.0186\pm 0.0057$ \\ \hline
            CoCaBO-0.0         & $  +1.4986\pm 0.2012$ \\
            CoCaBO-0.5         & $  +1.3720\pm 0.2110$ \\
            CoCaBO-1.0         & $  +1.8114\pm 0.2168$ \\
            \hline
        \end{tabular}
    \end{minipage}
    \label{supp:fig:exp-ackley5c}
\end{table}

\newpage

\subsection{SVM Hyperparameter Optimization}

\begin{table}[!h]
    \begin{minipage}{0.5\linewidth}
        \centering
        \includegraphics[width=0.9\columnwidth]{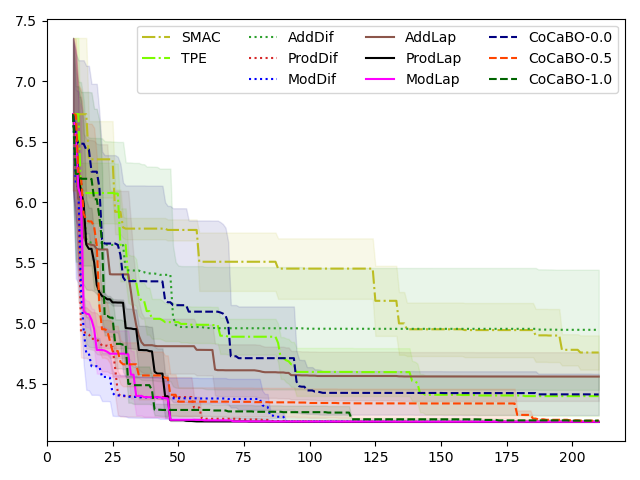}
    \end{minipage}
    \hfill
    \begin{minipage}{0.5\linewidth}
        \centering
        \begin{tabular}{lc}\hline
            Method             & Mean$\pm$Std.Err.     \\ \hline
            SMAC               & $  +4.7588\pm 0.1414$ \\
            TPE                & $  +4.3986\pm 0.1632$ \\ \hline
            AddDif             & $  +4.9463\pm 0.4960$ \\
            ProdDif            & $  +4.1857\pm 0.0017$ \\
            ModDif             & $  +4.1876\pm 0.0012$ \\ \hline
            AddLap             & $  +4.5600\pm 0.2014$ \\
            ProdLap            & $  +4.1856\pm 0.0012$ \\
            ModLap             & $  +4.1864\pm 0.0015$ \\ \hline
            CoCaBO-0.0         & $  +4.4122\pm 0.1703$ \\
            CoCaBO-0.5         & $  +4.1957\pm 0.0040$ \\
            CoCaBO-1.0         & $  +4.1958\pm 0.0037$ \\
            \hline
        \end{tabular}
    \end{minipage}
    \label{supp:fig:exp-svm}
\end{table}

\subsection{XGBoost Hyperparameter Optimization}

\begin{table}[!h]
    \begin{minipage}{0.5\linewidth}
        \centering
        \includegraphics[width=0.9\columnwidth]{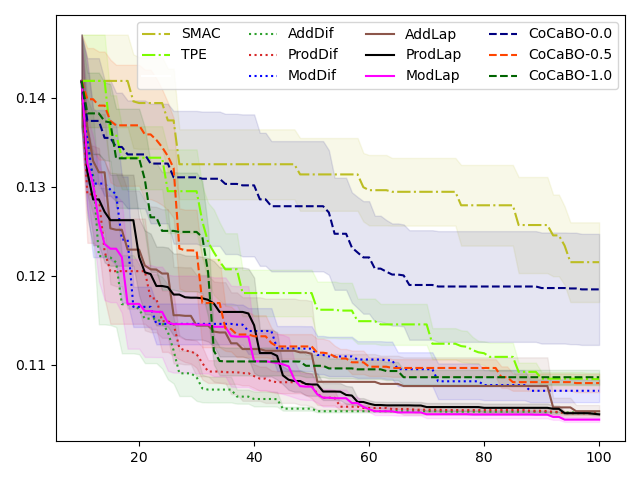}
    \end{minipage}
    \hfill
    \begin{minipage}{0.5\linewidth}
        \centering
        \begin{tabular}{lc}\hline
            Method             & Mean$\pm$Std.Err.     \\ \hline
            SMAC               & $  +0.1215\pm 0.0045$ \\
            TPE                & $  +0.1084\pm 0.0007$ \\ \hline
            AddDif             & $  +0.1046\pm 0.0001$ \\
            ProdDif            & $  +0.1045\pm 0.0003$ \\
            ModDif             & $  +0.1071\pm 0.0013$ \\ \hline
            AddLap             & $  +0.1048\pm 0.0007$ \\
            ProdLap            & $  +0.1044\pm 0.0001$ \\
            ModLap             & $  +0.1038\pm 0.0003$ \\ \hline
            CoCaBO-0.0         & $  +0.1184\pm 0.0062$ \\
            CoCaBO-0.5         & $  +0.1079\pm 0.0010$ \\
            CoCaBO-1.0         & $  +0.1086\pm 0.0008$ \\
            \hline
        \end{tabular}
    \end{minipage}
    \label{supp:fig:exp-xgboost}
\end{table}

\subsection{Joint Optimization of SGD hyperparameters and architecture.}

\begin{table}[!h]
    \begin{minipage}{0.5\linewidth}
        \centering
        \includegraphics[width=0.9\columnwidth]{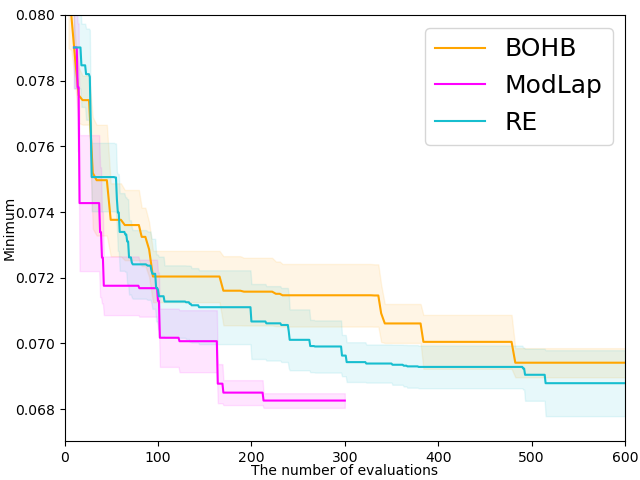}
    \end{minipage}
    \hfill
    \begin{minipage}{0.5\linewidth}
        \centering
        \begin{tabular}{lcc}\hline
            Method(\#Eval.) & Mean$\pm$Std.Err.     \\ \hline
            BOHB(200)       & \num{7.158E-02}$\pm$\num{1.0303E-03} \\ 
            BOHB(230)       & \num{7.151E-02}$\pm$\num{9.8367E-04} \\
            BOHB(600)       & \num{6.941E-02}$\pm$\num{4.4320E-04} \\ \hline
            RE(200)         & \num{7.067E-02}$\pm$\num{1.1417E-03} \\
            RE(230)         & \num{7.061E-02}$\pm$\num{1.1329E-03} \\
            RE(400)         & \num{6.929E-02}$\pm$\num{6.4804E-04} \\
            RE(600)         & \num{6.879E-02}$\pm$\num{1.0039E-03} \\ \hline
            ModLap(200)     & \num{6.850E-02}$\pm$\num{3.7914E-04} \\
            ModLap(230)     & \num{6.826E-02}$\pm$\num{2.2317E-04} \\
            ModLap(300)     & \num{6.826E-02}$\pm$\num{2.2317E-04} \\ \hline
        \end{tabular}
    \end{minipage}
    \label{supp:fig:exp-joint-nas}
\end{table}

\section{Experiment: Ablation Study}
\label{supp:sec:ablation_study}

We run a regression task on 3 different UCI datasets.
\begin{table}[!h]
    \caption{Regression Datasets}
    \centering
    \begin{tabular}{l|r|r|r} \hline
        Dataset                & \# of points & Continuous Dim. & Categorical Dim.\\ \hline
        Meta-data              &         528 &              16 &         3\\
        Servo                  &         167 &             ~~2 &          2\\
        Optical Intercon. Net. &         640 &             ~~2 &          2\\
        \hline
    \end{tabular}
\end{table}

On 20 different random splits (training:test=8:2), negative log likelihood(NLL) and RMSE on test set are reported in Table~\ref{supp:tab:regression}.

\begin{table}[!h]
    \caption{Regression}
    \centering
    \begin{tabular}{c|c|c|c}\hline
        NLL      & Meta-data             & Servo                     & Optical Intercon. Net.\\ \hline
        AddDif   & $  16.0224\pm 3.9906$ & $~~~~4.2362\pm~~~~0.6115$ & $~~7.5504\pm 0.4867$\\
        ProdDif  & $ ~~9.5198\pm 3.7116$ & $~~~~0.9579\pm~~~~0.4758$ & $~~0.2132\pm 0.2050$\\
        ModDif   & $ ~~5.9377\pm 1.9872$ & $  503.9973\pm  486.4679$ & $  10.0005\pm 0.2934$\\
        AddLap   & $ ~~1.6805\pm 0.1847$ & $~~~~3.7083\pm~~~~0.5001$ & $~~7.5568\pm 0.4897$\\
        ProdLap  & $ ~~1.3236\pm 0.3539$ & $~~~~0.7008\pm~~~~0.3385$ & $~~0.2135\pm 0.1928$\\
        ModLap   & $ ~~1.1218\pm 0.2987$ & $~~~~1.0790\pm~~~~0.4607$ & $~~0.1521\pm 0.2265$\\
        \hline
    \end{tabular}
    \\
    
    \begin{tabular}{c|c|c|c}\hline
        RMSE     & Meta-data             & Servo                     & Optical Intercon. Net.\\ \hline
        AddDif   & $ ~~1.0223\pm 0.1601$ & $ ~~~~0.5696\pm ~~~~0.0310$ & $ ~~0.2577\pm 0.0052$\\
        ProdDif  & $ ~~1.1537\pm 0.1654$ & $ ~~~~0.3023\pm ~~~~0.0408$ & $ ~~0.1413\pm 0.0060$\\
        ModDif   & $ ~~1.4074\pm 0.2027$ & $ ~~~~0.7308\pm ~~~~0.0910$ & $ ~~0.7881\pm 0.0069$\\
        AddLap   & $ ~~1.0199\pm 0.1588$ & $ ~~~~0.5709\pm ~~~~0.0311$ & $ ~~0.2577\pm 0.0052$\\
        ProdLap  & $ ~~1.0898\pm 0.1642$ & $ ~~~~0.2971\pm ~~~~0.0405$ & $ ~~0.1417\pm 0.0059$\\
        ModLap   & $ ~~1.0920\pm 0.1626$ & $ ~~~~0.3046\pm ~~~~0.0412$ & $ ~~0.1400\pm 0.0063$\\
        \hline
    \end{tabular}
    \label{supp:tab:regression}
\end{table}

In terms of NLL, which takes into account uncertainty, ModLap is the best in Meta-data and Optical Intercon. Net. 
In Servo, ProdLap/ProdDif perform the best, so we conjecture that this dataset has an approximate product structure. 
In terms of RMSE, ModLap and ProdLap are equally good. 
We conclude that the frequency modulation has the benefit beyond the addition/product of good basis kernels.
Also, the importance of respecting the similarity measure behavior is observed on the regression task.


\end{document}